\documentclass[11pt,letterpaper]{article}
\usepackage[T1]{fontenc}
\usepackage[utf8]{inputenc}
\usepackage{lmodern}
\usepackage[margin=1in]{geometry}
\usepackage{amsmath,amssymb,amsthm,mathtools}
\usepackage{microtype}
\usepackage{booktabs,array,tabularx}
\usepackage[round,authoryear]{natbib}
\usepackage{xurl}
\usepackage[hidelinks]{hyperref}
\usepackage{enumitem}
\usepackage{fancyhdr}
\usepackage{needspace}
\hypersetup{pdftitle={Ordinary Nonconvex SGD under Distance-Dependent Moments},pdfauthor={Wei Biao Wu},pdfsubject={Minimax BG-0 stationarity and finite-moment confidence bounds for single-sample SGD}}
\numberwithin{equation}{section}
\newtheorem{theorem}{Theorem}[section]
\newtheorem{proposition}[theorem]{Proposition}
\newtheorem{lemma}[theorem]{Lemma}
\newtheorem{corollary}[theorem]{Corollary}
\theoremstyle{definition}
\newtheorem{assumption}[theorem]{Assumption}
\newtheorem{example}[theorem]{Example}
\theoremstyle{remark}
\newtheorem{remark}[theorem]{Remark}
\newcommand{\R}{\mathbb R}
\newcommand{\E}{\mathbb E}
\newcommand{\Prob}{\mathbb P}
\newcommand{\F}{\mathcal F}
\newcommand{\1}{\mathbf 1}
\newcommand{\norm}[1]{\left\lVert #1\right\rVert}
\newcommand{\ip}[2]{\left\langle #1,#2\right\rangle}
\newcommand{\cprob}{\mathfrak c_p}
\newcommand{\Dprob}{\mathfrak D_p}
\newcommand{\eprob}{\mathfrak e_p}
\newcommand{\Hprob}{\mathfrak h_p}
\newcommand{\Kprob}{\mathfrak k_p}
\DeclareMathOperator{\dist}{dist}
\newcommand{\op}{\mathrm{op}}
\newcommand{\eps}{\varepsilon}

\allowdisplaybreaks[2]
\title{\textbf{Ordinary Nonconvex SGD under\newline Distance-Dependent Moments}\\[0.35em]
\large Finite-Horizon Stationarity and Nagaev Bounds}
\author{Wei Biao Wu}
\date{October 6, 2026\\[0.15em]\small Enhanced revision}
\begin{document}
\maketitle
\begin{abstract}
Uniform noise-moment bounds exclude stochastic gradients whose variability increases with the iterate. We study ordinary, single-sample stochastic gradient descent for smooth, lower-bounded, possibly nonconvex objectives under distance-dependent moments. A descent--displacement argument gives $T^{-1/3}$ expected average squared-gradient stationarity and an expected maximal-displacement bound under marginal second-moment growth alone. The finite-horizon oracle complexity matches the known smooth Blum--Gladyshev lower bound, including the $Lb_2\Delta^3\varepsilon^{-6}$ and $L\Delta\sigma^2\varepsilon^{-4}$ stochastic terms. A decreasing-stepsize extension gives a single horizon-free run with expected stationarity $O(T^{-1/3}(\log T)^\kappa)$ for every $\kappa>1/3$, as well as almost-sure convergence of gradient norms to zero. Under conditional $p$th moments, $p>2$, centered noise-energy concentration and predictable stopping of objective-plus-gradient energy sharpen the confidence bound to a variance baseline $L\alpha\sigma^2$, a smaller logarithmic term $\sigma^2\log(1/\delta)/T$, and a polynomial rare-shock term. In particular, a constant step within each run of order $T^{-2/3}$ gives $T^{-1/3}$ stationarity with failure probability $T^{-r}$ for $0<r\le p/2-1$, without a logarithmic rate loss. Objective-gap moment growth similarly gives root-$T$ stationarity at these confidence levels. No clipping, normalization, momentum, increasing batch size, or bounded-iterate assumption is imposed. Stochastic Lipschitz and random-design examples verify the assumptions, including models with nonunique unbounded minimizer sets. The broad second-moment optimality statement is distinguished from the smaller mean-square-smooth oracle class.
\end{abstract}
\noindent\textbf{Keywords:} stochastic gradient descent; nonconvex optimization; distance-dependent variance; stochastic Lipschitz continuity; martingale concentration; Fuk--Nagaev inequality.

\noindent\textbf{2020 Mathematics Subject Classification.} Primary 62L20; Secondary 60E15, 60G42, 90C15, 90C26.

\section{Introduction}\label{sec:introduction}
The stochastic gradient available at a parameter $x$ need not have a noise distribution whose scale is bounded independently of $x$. Random-design least squares provides a simple example: its gradient contains a random matrix multiplied by the current parameter. More generally, an $L^p$-Lipschitz stochastic-gradient field has noise moments that can grow linearly with the distance from a reference point. Fresh independent observations guarantee conditional unbiasedness along an adaptive optimization trajectory, but do not turn this state-dependent scale into a deterministic uniform bound.

Stochastic gradient methods are central to large-scale learning \citep{bottou2018}. This distinction matters for nonconvex optimization. In a strongly convex model, control of objective values or distance to a minimizer may supply the missing stability. For a general smooth, lower-bounded objective, a small objective value need not control the distance traveled, and the objective may be constant in directions along which the stochastic-gradient variance increases. A proof that simply substitutes the moment at initialization for a global noise envelope is therefore invalid.

We consider the unmodified recursion
\begin{equation}\label{eq:sgd-intro}
 x_{t+1}=x_t-\alpha G(x_t,Z_t),\qquad t=1,\ldots,T,
\end{equation}
with one fresh independent observation $Z_t$ at each iteration. The finite-horizon confidence theorems use a stepsize constant within the run, selected as a function of its horizon and desired confidence. Section~\ref{sec:single-run} separately treats one infinite run with deterministic decreasing stepsizes. Our principal criterion is
\begin{equation}\label{eq:criterion-intro}
 A_T=\frac1T\sum_{t=1}^T\norm{\nabla F(x_t)}^2,
 \qquad F(x)=\E f(x,Z)
\end{equation}
when a differentiable random loss representation is available. The results require an unbiased stochastic-gradient oracle, not an interchange of differentiation and expectation without justification. We also state the theorems directly for martingale-difference errors, so that the exact role of independent sampling is transparent.

The expectation results need only a marginal second-moment growth bound. The confidence results use stronger conditional bounds: the conditional second moment and the squared conditional $L^p$ norm of the error are at most affine functions of $\norm{x_t-x_1}^2$. A two-point stochastic Lipschitz assumption is sufficient, but is not needed by the optimization argument. This distinction enlarges the theorem's scope and identifies precisely what enters the proof.

\subsection{Results and mechanism}
Our first result uses only marginal second-moment growth and conditional centering. It couples expected descent with the exact displacement decomposition
\begin{equation}\label{eq:displacement-intro}
 x_{k+1}-x_1=-\alpha\sum_{t=1}^k\nabla F(x_t)
               -\alpha\sum_{t=1}^k\xi_t,
 \qquad \xi_t=G(x_t,Z_t)-\nabla F(x_t).
\end{equation}
The deterministic gradient sum is controlled by Cauchy--Schwarz, while martingale orthogonality controls the second moment of the noise sum. The resulting feedback term can be absorbed when a quantity proportional to
\begin{equation}\label{eq:h-intro}
 h_T=\alpha^2T(1+L\alpha T)
\end{equation}
is sufficiently small. We obtain the familiar bounded-variance form, up to a factor of two from the absorption step:
\[
 \E A_T\ \le\ \frac{4\Delta}{\alpha T}+2L\alpha\sigma^2,
 \qquad \Delta=F(x_1)-F_*,
\]
but with an additional horizon-dependent admissibility condition. Here $\sigma^2$ is an intercept in the distance-dependent variance bound, not a uniform variance bound over $\R^d$. A step $\alpha\asymp T^{-2/3}$ makes the feedback condition hold and yields $\E A_T=O(T^{-1/3})$. More precisely, Corollary~\ref{cor:complexity} gives a deterministic oracle budget of order
\begin{equation}\label{eq:complexity-intro}
 1+\frac{L\Delta}{\eps^2}
   +\frac{L\Delta\sigma^2}{\eps^4}
   +\frac{Lb_2\Delta^3}{\eps^6}
\end{equation}
in the small-accuracy regime, for a uniformly randomized output satisfying $\E\|\nabla F(\widehat x)\|^2\le\eps^2$. The two stochastic terms match the lower bound of \citet[Theorem~1]{fazla2026}; Section~\ref{sec:minimax} gives the precise oracle-class and accuracy comparison.

The same second-moment argument also controls the largest displacement during the run. After closing the marginal second-moment feedback, Doob's $L^2$ maximal inequality gives
\[
 \E\max_{1\le t\le T+1}\norm{x_t-x_1}^2
 \le20\alpha T\Delta+10h_T\sigma^2
\]
under the unchanged restriction $b_2h_T\le1/4$. In particular, moving the maximum inside the expectation requires neither a higher moment nor a smaller admissible stepsize. The theorem retains the sharper bound for the maximum of the marginal second moments as well.

The same mechanism has a horizon-free extension. With
\[
 \eta_t=a_0t^{-2/3}\{\log(et)\}^{-\kappa},\qquad \kappa>1/3,
\]
a sufficiently small $a_0$ gives $\E A_T=O(T^{-1/3}(\log T)^\kappa)$ for every prefix of one run. The weighted noise martingale converges, the total weighted gradient energy is finite, and $\|\nabla F(x_t)\|\to0$ almost surely. The logarithmic loss here is the price of the displayed sufficient single-run schedule; it is distinct from confidence dependence in a horizon-tuned run.

Our confidence analysis has two levels. Proposition~\ref{prop:uniform} and Theorem~\ref{thm:main} provide explicit all-horizon bounds using deterministic conditional envelopes and a derived distance-localization radius. Their sharper counterparts, Proposition~\ref{prop:sharp-uniform} and Theorem~\ref{thm:sharp-main}, center the squared-noise energy and stop the objective gap plus accumulated squared-gradient energy. The latter controls the \emph{sum} of the adaptive transform's squared coefficients, not only their maximum. For $L\alpha T\ge1$, the refined stationarity bound has the form
\begin{equation}\label{eq:sharp-intro}
 A_T\le C_p\left\{\frac{\Delta}{\alpha T}
       +L\alpha\sigma^2+\frac{\sigma^2\log(6/\delta)}{T}
       +L\alpha\nu_p^2T^{2/p-1}\delta^{-2/p}\right\}
\end{equation}
with probability at least $1-\delta$, under an explicit confidence-sensitive feedback condition. The localization radius is proved from the recursion, not imposed on the optimizer. The baseline theorem remains available for shorter runs.

Writing $\rho=T^{2/p-1}\delta^{-2/p}$ and $\ell=\log(6/\delta)$, the refined condition controls
\[
 \Xi_T=\alpha^2Tb_2\ell+L\alpha^3T^2(b_2+b_p\rho).
\]
In particular, when $\delta=T^{-r}$ and $0<r\le p/2-1$, a sufficiently small $\alpha=a_0T^{-2/3}$ gives $A_T=O(T^{-1/3})$ without a logarithmic loss. For example, fourth moments permit failure probability $T^{-1}$ at this rate. These confidence assertions concern individual fixed-horizon runs, not simultaneous events for the decreasing-stepsize process.

Under objective-gap moment growth, the analogous feedback condition is
$L\alpha^2T(\gamma_2+\gamma_p\rho)+\alpha\gamma_2\ell\lesssim_p1$.
A step of order $T^{-1/2}$ then gives root-$T$ stationarity with failure probability $T^{-r}$ throughout the same range. The geometric verification accommodates an unbounded minimizer set when the noise is uniformly controlled on that set. An additive-shock calculation explains why a polynomial confidence component remains necessary for unchanged SGD under finite moments.

\subsection{Relation to existing work}\label{sec:related}
\paragraph{Distance growth and fixed-horizon analysis.}
Distance-dependent moment conditions originate in stochastic approximation \citep{blum1954,gladyshev1965}. The BG-0 terminology records that the reference point is the initialization. \citet{alacaoglu2026} develop convex guarantees using Halpern anchoring and compare these growth assumptions with expected smoothness. Centered-noise and uncentered-gradient conditions must be distinguished; Proposition~\ref{prop:bg-equivalence} gives their relation under objective smoothness. For convex composite optimization, \citet[Theorem~8]{domke2023} already use a fixed horizon to control quadratically growing estimator moments. The difficulty of a direct nonconvex analysis, without a regularization device, is identified in Section~5 of \citet{alacaoglu2026}. Theorem~\ref{thm:expectation} answers its fixed-horizon version. Section~\ref{sec:single-run} adds a direct horizon-free result, with logarithmic exponent $\kappa>1/3$; an exactly log-free horizon-free rate remains outside the proved results.

\paragraph{Oracle complexity and the role of the update.}
\citet{ghadimi2013} give foundational finite-variance nonconvex stationarity guarantees. The bounded-variance oracle lower bounds of \citet{arjevani2023} distinguish objective smoothness from mean-square smoothness of the stochastic oracle. Under BG-0, \citet{fazla2026} establish the corresponding $\eps^{-6}$ lower bound and achieve this order with distance-dependent dynamic batching. \citet[Corollary~3.1]{upadhyay2026} obtain the same order for an expected gradient-norm criterion using normalized momentum. Our second-moment result attains the BG-0 stochastic complexity with exactly one fresh oracle call and the ordinary SGD update at every iteration. The direct descent--displacement closure, rather than a new growth assumption or a new worst-case exponent, is the algorithm-specific contribution.

\paragraph{Additional geometry and stochastic smoothness.}
The objective-gap second-moment condition in Section~\ref{sec:gap} lies in the expected-smoothness/ABC framework of \citet{khaled2023}; our result there is a finite-$p$ confidence refinement. Adaptive methods under gradient-dependent affine variance, such as the AdaGrad-Norm analysis of \citet{faw2022}, use a different growth condition and update. The stochastic $L^p$-Lipschitz condition of \citet[Assumption~2 in the public draft]{li2026} supplies our distance-growth assumptions from a moment at one point. Their moment-contraction analysis also uses strong convexity and a noise-sensitive stepsize. Constant-step invariant-law and averaged-iterate analyses under additional geometry appear in \citet{dieuleveut2020}. For $p>2$, stochastic $L^p$-Lipschitz continuity implies mean-square smoothness. The faster variance-reduced BG-0 methods of \citet{fazla2026} exploit that extra structure, so the broad-class optimality result is not an optimality assertion for this smaller oracle class.

\paragraph{Heavy tails and high-probability optimization.}
\citet{zhang2020} connect heavy-tailed stochastic gradients with attention models and establish optimization upper and lower bounds in a finite-$q$ regime with $1<q\le2$. Their Theorem~6 gives an $\Omega(\eps^{-(3q-2)/(q-1)})$ nonconvex query lower bound under their uniform central-$q$-moment assumptions and stated parameter regime; this differs from the distance-growth, finite-variance BG-0 lower bound. Clipping-based high-probability analyses include convex accelerated methods \citep{gorbunov2020}, clipping with normalization and momentum for nonconvex objectives \citep{cutkosky2021}, and finite-central-moment results \citep{sadiev2023,nguyen2023}. \citet{liuzhou2025} obtain clipping-free heavy-tail guarantees using batched normalized momentum, not the ordinary update. Under sub-Gaussian noise, \citet{liunguyen2023} give high-probability results for SGD and adaptive methods. The cited robust methods can obtain logarithmic confidence dependence under their respective uniform moment assumptions, whereas our unchanged update retains a polynomial rare-shock term. Our centered-energy refinement places the logarithmic variance cost at order $\sigma^2\log(1/\delta)/T$, in addition to the baseline $L\alpha\sigma^2$. These works separate two issues that also matter here: an infinite variance caused by $q<2$ is different from finite conditional variance that is unbounded over the parameter space, and robust updates have different confidence costs from unchanged SGD. Our $p>2$ theorem keeps the ordinary update and separates a variance-scale contribution from a polynomial rare-shock contribution under distance-dependent moments.

\paragraph{Probability tools and Nagaev bounds for SGD.}
The exponential-plus-polynomial approach originates with \citet{fuk1971,nagaev1979}. Martingale developments include \citet{pinelis1994,fan2017,rio2017}; Doob's classical maximal inequality \citep{doob1953} gives the second-moment trajectory bound, while the finite-$p$ concentration input is the maximal Hilbert-space specialization of \citet[Theorem~2.2]{mollenhauer2025}. For averaged SGD in a linear model, \citet{zhu2022} already obtain Nagaev estimation bounds and demonstrate the necessity of a polynomial confidence contribution. Our optimization criterion is instead nonconvex average squared-gradient stationarity. The new extension derives deterministic conditional envelopes from iterate-dependent growth by closing the displacement bound, rather than assuming bounded iterates. Section~\ref{sec:provenance} records the supplied unpublished source of the deterministic-envelope and quadratic-shock ingredients, which are included with full proofs and are not counted among the new results.

\paragraph{Organization.}
Section~\ref{sec:setup} states the assumptions and verifies stochastic Lipschitz oracles. Section~\ref{sec:expectation} proves the second-moment stationarity and maximal-displacement theorem and its decreasing-stepsize extension. Sections~\ref{sec:probability} and~\ref{sec:uniform} develop the probability and deterministic-envelope ingredients. Section~\ref{sec:main} gives the baseline and energy-refined distance-dependent Nagaev theorems, and Section~\ref{sec:rates} gives their stepsize and confidence consequences. Section~\ref{sec:gap} treats objective-gap growth. Sections~\ref{sec:examples} and~\ref{sec:lower} give examples, a stability counterexample, and the additive rare-shock calculation. Section~\ref{sec:discussion} records the scope and limitations.

\section{Setup and moment assumptions}\label{sec:setup}
Let $(\Omega,\F,\Prob)$ carry an increasing filtration $(\F_t)_{t=0}^T$. Write $\E_{t-1}Y=\E(Y\mid\F_{t-1})$. Norms without a subscript are Euclidean, matrix operator norms are denoted by $\norm{\cdot}_{\op}$, and $\norm{Y}_q=(\E\norm{Y}^q)^{1/q}$. All almost-sure conditional inequalities are understood for the finitely many indices in the run. Conditional moments of nonnegative variables are initially interpreted in the extended sense; Lemma~\ref{lem:integrability} then establishes the finite-horizon integrability used in the proofs.

\begin{assumption}[Objective and recursion]\label{ass:objective}
The function $F:\R^d\to\R$ is differentiable, its gradient is $L$-Lipschitz for some $L>0$, and $F(x)\ge F_*$ for a finite constant $F_*$. The initialization $x_1$ is deterministic. For a deterministic stepsize $\alpha>0$,
\begin{equation}\label{eq:sgd}
 x_{t+1}=x_t-\alpha(g_t+\xi_t),\qquad g_t=\nabla F(x_t),
 \qquad 1\le t\le T,
\end{equation}
where $x_t$ is $\F_{t-1}$-measurable, $\xi_t$ is $\F_t$-measurable, and $\E_{t-1}\xi_t=0$.
\end{assumption}
We use $\Delta=F(x_1)-F_*$ and $A_T=T^{-1}\sum_{t=1}^T\norm{g_t}^2$ throughout. The lower bound $F_*$ need not be the attained minimum, and no set of minimizers is required. Every upper bound remains valid when $\Delta$ is replaced by a known deterministic upper bound on $F(x_1)-F_*$.

\begin{assumption}[Distance-dependent second moment]\label{ass:second}
There are finite $\sigma^2,b_2\ge0$ such that
\begin{equation}\label{eq:second-growth}
 \E_{t-1}\norm{\xi_t}^2\le \sigma^2+b_2\norm{x_t-x_1}^2
 \qquad\hbox{almost surely},\quad 1\le t\le T.
\end{equation}
\end{assumption}
\begin{assumption}[Distance-dependent finite-$p$ moments]\label{ass:p}
For a fixed $p>2$, Assumption~\ref{ass:second} holds and there are finite $\nu_p^2,b_p\ge0$ such that
\begin{equation}\label{eq:p-growth}
 \bigl(\E_{t-1}\norm{\xi_t}^p\bigr)^{2/p}
 \le \nu_p^2+b_p\norm{x_t-x_1}^2
 \qquad\hbox{almost surely},\quad 1\le t\le T.
\end{equation}
\end{assumption}
The quantities $b_2,b_p$ multiply a squared distance. The quantity $\nu_p^2$ is a squared $L^p$ scale; it is not a $p$th moment. Both intercepts may be chosen as upper bounds rather than exact moments. The constants are deterministic, and rate statements hold them fixed as the horizon increases.

\paragraph{Marginal second-moment condition (M2).}
For the expectation results only, it suffices that the errors are square-integrable martingale differences and
\begin{equation}\label{eq:marginal-growth}
 \E\|\xi_t\|^2\le\sigma^2+b_2\E\|x_t-x_1\|^2.
 \tag{M2}
\end{equation}
The constants are deterministic. Assumption~\ref{ass:second} implies (M2) by total expectation and Lemma~\ref{lem:integrability}. Conversely, (M2) does not supply deterministic conditional envelopes. In the marginal formulation, square integrability of the errors is an explicit premise; the recursion and objective smoothness then give square integrability of the iterates and integrability of their objective values. The finite-$p$ confidence results continue to require their displayed conditional assumptions.

\subsection{Independent samples and stochastic Lipschitz continuity}
Let $Z_t$ be iid with law $P_Z$, independent of $\F_{t-1}$, and let $G:\R^d\times\mathcal Z\to\R^d$ be jointly measurable. Suppose, for every deterministic $x$,
\begin{equation}\label{eq:oracle-unbiased}
 \E_ZG(x,Z)=\nabla F(x),\qquad e(x,z)=G(x,z)-\nabla F(x).
\end{equation}
The oracle representation in~\eqref{eq:sgd} is $\xi_t=e(x_t,Z_t)$.

\begin{proposition}[Verification from stochastic Lipschitz continuity]\label{prop:lipschitz}
Suppose that for $p>2$,
\begin{equation}\label{eq:stochastic-lipschitz}
 \norm{G(x_1,Z)}_p<\infty,\qquad
 \norm{G(x,Z)-G(y,Z)}_p\le L_{G,p}\norm{x-y}
 \quad(x,y\in\R^d).
\end{equation}
For $q=2,p$, let $L_{G,q}$ be an admissible $L^q$-Lipschitz constant, put
\[
 s_q=\norm{e(x_1,Z)}_q,\qquad
 K_2=L_{G,2},\qquad K_p=L_{G,p}+L,
\]
and take $L_{G,2}=L_{G,p}$ when no sharper second-moment constant is available. Then
\begin{equation}\label{eq:noise-lipschitz}
 \norm{e(x,Z)-e(y,Z)}_q\le K_q\norm{x-y},\qquad
 \norm{e(x,Z)}_q\le s_q+K_q\norm{x-x_1}.
\end{equation}
Assumption~\ref{ass:p} holds with
\begin{equation}\label{eq:lipschitz-constants}
 \sigma^2=2s_2^2,\quad b_2=2K_2^2,\qquad
 \nu_p^2=2s_p^2,\quad b_p=2K_p^2.
\end{equation}
More generally, for any $\chi>0$, the second-moment constants can be chosen as
\begin{equation}\label{eq:lipschitz-tradeoff}
 \sigma^2=(1+\chi)s_2^2,\qquad
 b_2=(1+\chi^{-1})L_{G,2}^2.
\end{equation}
The finite moment in~\eqref{eq:stochastic-lipschitz} may instead be imposed at any deterministic reference point.
\end{proposition}
\begin{proof}
Monotonicity of probability-space $L^q$ norms gives the second-moment Lipschitz inequality from the $p$th-moment inequality. For $q=2$, let $Y=G(x,Z)-G(y,Z)$. Unbiasedness gives $e(x,Z)-e(y,Z)=Y-\E Y$, and the Hilbert-space centering identity gives
\[
 \|e(x,Z)-e(y,Z)\|_2^2
 =\E\|Y\|^2-\|\E Y\|^2
 \le L_{G,2}^2\|x-y\|^2.
\]
For $q=p$, the triangle inequality and objective smoothness give
\[
 \norm{e(x,Z)-e(y,Z)}_p
 \le\norm{G(x,Z)-G(y,Z)}_p+\norm{\nabla F(x)-\nabla F(y)}
 \le K_p\norm{x-y}.
\]
Set $y=x_1$ and use the triangle inequality once more to obtain~\eqref{eq:noise-lipschitz}. Squaring and using $(a+b)^2\le2a^2+2b^2$ gives the displayed pointwise moment-growth bounds. The inequality $(a+b)^2\le(1+\chi)a^2+(1+\chi^{-1})b^2$ proves~\eqref{eq:lipschitz-tradeoff}.

Independence and joint measurability imply, for predictable $x_t$,
\begin{equation}\label{eq:conditional-substitution}
 \E_{t-1}e(x_t,Z_t)=0,\qquad
 \E_{t-1}\norm{e(x_t,Z_t)}^q
   =\left.\E_Z\norm{e(x,Z)}^q\right|_{x=x_t}.
\end{equation}
The identity for nonnegative functions follows first for product indicator functions and then by the monotone class theorem. For vector means it follows by integrability and componentwise application. Finite-horizon integrability also follows by the induction in Lemma~\ref{lem:integrability}. Substitution into the pointwise inequalities proves~\eqref{eq:lipschitz-constants}. Finally,
\[
 \norm{G(x_1,Z)}_p
 \le \norm{G(x_{\rm ref},Z)}_p+L_{G,p}\norm{x_1-x_{\rm ref}}
\]
transfers any finite reference-point moment to the initialization.
\end{proof}
\begin{remark}\label{rem:no-uniform}
The quantifiers in~\eqref{eq:stochastic-lipschitz} do not imply $\sup_x\norm{e(x,Z)}_p<\infty$. They control increments of the random field, not its absolute size over an unbounded parameter space. Nor does fresh-sample independence remove $\norm{x_t-x_1}$ from~\eqref{eq:conditional-substitution}. In addition, a two-point increment condition alone does not imply a finite moment at any point: an arbitrary additive random vector can be added without changing the increments. The single-point moment condition is indispensable for this verification.
\end{remark}

\subsection{Integrability and elementary objective bounds}
\begin{lemma}[Finite-horizon integrability]\label{lem:integrability}
Under Assumptions~\ref{ass:objective} and~\ref{ass:second}, $x_t,g_t,\xi_t$ are square-integrable for every finite index in the run, and $F(x_t)$ is integrable. Under Assumption~\ref{ass:p}, these vectors belong to $L^p$ as well.
\end{lemma}
\begin{proof}
Write $d_t=x_t-x_1$. Objective smoothness gives
\begin{equation}\label{eq:gradient-linear}
 \norm{g_t}\le \norm{\nabla F(x_1)}+L\norm{d_t}.
\end{equation}
For $q=2$ use $\omega_q=\sigma^2,c_q=b_2$, and for $q=p$ use $\omega_q=\nu_p^2,c_q=b_p$. The conditional growth inequality and Minkowski give
\begin{align*}
 \norm{\xi_t}_q
 &=\left\|\bigl(\E_{t-1}\norm{\xi_t}^q\bigr)^{1/q}\right\|_q\\
 &\le \left\|\sqrt{\omega_q+c_q\norm{d_t}^2}\right\|_q
 \le \sqrt{\omega_q}+\sqrt{c_q}\,\norm{d_t}_q.
\end{align*}
Starting with $d_1=0$, the recursion implies
\[
 \norm{d_{t+1}}_q
 \le\{1+\alpha(L+\sqrt{c_q})\}\norm{d_t}_q
       +\alpha\{\norm{\nabla F(x_1)}+\sqrt{\omega_q}\}.
\]
An induction proves finiteness at every finite time. No uniform-in-$t$ bound is asserted by this induction. The lower bound on $F$ and the smooth upper estimate
\[
 F(x_t)\le F(x_1)+\ip{\nabla F(x_1)}{d_t}+\frac L2\norm{d_t}^2
\]
then show $F(x_t)\in L^1$.
\end{proof}

\begin{lemma}[Self-bounding gradient]\label{lem:self-bound}
For every $x\in\R^d$,
\begin{equation}\label{eq:self-bound}
 \norm{\nabla F(x)}^2\le 2L\{F(x)-F_*\}.
\end{equation}
\end{lemma}
\begin{proof}
The descent lemma at $y=x-L^{-1}\nabla F(x)$ gives
\[
 F_*\le F(y)\le F(x)-L^{-1}\norm{\nabla F(x)}^2
                  +(2L)^{-1}\norm{\nabla F(x)}^2.
\]
Rearrangement proves the claim. Attainment of $F_*$ is not used.
\end{proof}

\begin{proposition}[Centered and uncentered distance growth]\label{prop:bg-equivalence}
For an unbiased iid oracle, suppose
\[
 \E_Z\norm{e(x,Z)}^2\le \sigma^2+b\norm{x-x_1}^2.
\]
Then
\begin{equation}\label{eq:bg-uncentered}
 \E_Z\norm{G(x,Z)}^2
 \le \sigma^2+2\norm{\nabla F(x_1)}^2
            +(b+2L^2)\norm{x-x_1}^2.
\end{equation}
Conversely, an uncentered bound $\E_Z\norm{G(x,Z)}^2\le a+c\norm{x-x_1}^2$ implies the same bound for $\E_Z\norm{e(x,Z)}^2$.
\end{proposition}
\begin{proof}
Unbiasedness gives the orthogonal decomposition
\[
 \E_Z\norm{G(x,Z)}^2=\norm{\nabla F(x)}^2+\E_Z\norm{e(x,Z)}^2.
\]
The gradient estimate~\eqref{eq:gradient-linear}, with $(a+b)^2\le2a^2+2b^2$, proves~\eqref{eq:bg-uncentered}. Dropping the nonnegative squared-gradient term proves the converse.
\end{proof}

\section{A second-moment stationarity theorem}\label{sec:expectation}
The first theorem requires only the marginal second-moment condition (M2), which is weaker than Assumption~\ref{ass:second}. In addition to expected stationarity, it controls both the maximum of the marginal mean-square displacements and the expectation of the squared maximum displacement. The latter follows from Doob's inequality after the marginal feedback is closed, so it uses the same stepsize restriction.

\begin{theorem}[Expected stationarity and maximal displacement]\label{thm:expectation}
Under Assumption~\ref{ass:objective} and (M2), define
\begin{equation}\label{eq:h}
 h_T=\alpha^2T(1+L\alpha T).
\end{equation}
If
\begin{equation}\label{eq:expect-condition}
 0<\alpha\le L^{-1},\qquad b_2h_T\le\frac14,
\end{equation}
then
\begin{equation}\label{eq:expect-bound}
 \boxed{\quad \E A_T\le\frac{4\Delta}{\alpha T}+2L\alpha\sigma^2.\quad}
\end{equation}
Moreover,
\begin{equation}\label{eq:expect-distance}
 \max_{1\le t\le T+1}\E\norm{x_t-x_1}^2
 \le8\alpha T\Delta+4h_T\sigma^2.
\end{equation}
In addition, the expected squared maximum displacement satisfies
\begin{equation}\label{eq:expect-max-distance}
 \boxed{\quad
 \mathcal R_T:=\E\max_{1\le t\le T+1}\norm{x_t-x_1}^2
 \le20\alpha T\Delta+10h_T\sigma^2.\quad}
\end{equation}
The bound~\eqref{eq:expect-distance} has smaller constants for the marginal second moments; \eqref{eq:expect-max-distance} places the maximum inside the expectation. All displayed numerical constants are independent of dimension.
\end{theorem}
\begin{proof}
Under (M2), square integrability of the errors and objective smoothness justify all expectations by induction through the recursion. Under the stronger conditional formulation, Lemma~\ref{lem:integrability} supplies the same conclusion. Define
\[
 S_T=\sum_{t=1}^T\E\norm{g_t}^2,\qquad
 V_T=\sum_{t=1}^T\E\norm{\xi_t}^2,\qquad
 R_T=\max_{1\le t\le T+1}\E\norm{x_t-x_1}^2.
\]
The smoothness inequality and conditional centering imply
\begin{align}
 \E_{t-1}F(x_{t+1})
 &\le F(x_t)-\alpha\norm{g_t}^2
              +\frac{L\alpha^2}{2}\E_{t-1}\norm{g_t+\xi_t}^2\notag\\
 &=F(x_t)-\alpha\left(1-\frac{L\alpha}{2}\right)\norm{g_t}^2
              +\frac{L\alpha^2}{2}\E_{t-1}\norm{\xi_t}^2.
 \label{eq:expected-descent}
\end{align}
Since $\alpha\le L^{-1}$, summation, total expectation, and $F(x_{T+1})\ge F_*$ give
\begin{equation}\label{eq:S-bound}
 S_T\le\frac{2\Delta}{\alpha}+L\alpha V_T.
\end{equation}
For $k\le T$, the exact identity~\eqref{eq:displacement-intro} and $(a+b)^2\le2a^2+2b^2$ yield
\[
 \E\norm{x_{k+1}-x_1}^2
 \le2\alpha^2k\sum_{t=1}^k\E\norm{g_t}^2
       +2\alpha^2\E\norm{\sum_{t=1}^k\xi_t}^2.
\]
For $s<t$, $\xi_s$ is $\F_{t-1}$-measurable and
$\E\ip{\xi_s}{\xi_t}=\E\ip{\xi_s}{\E_{t-1}\xi_t}=0$.
Thus martingale orthogonality gives
\begin{equation}\label{eq:R-S-V}
 R_T\le 2\alpha^2TS_T+2\alpha^2V_T.
\end{equation}
The marginal growth bound (M2) gives directly
\begin{equation}\label{eq:V-R}
 V_T\le T(\sigma^2+b_2R_T).
\end{equation}
Combining~\eqref{eq:S-bound}--\eqref{eq:V-R},
\begin{equation}\label{eq:R-feedback}
 R_T\le4\alpha T\Delta+2h_T(\sigma^2+b_2R_T).
\end{equation}
As $2b_2h_T\le1/2$, the last term can be moved to the left, proving~\eqref{eq:expect-distance}.

Substituting this displacement bound into~\eqref{eq:S-bound} and~\eqref{eq:V-R} gives
\begin{align}
 \E A_T
 &\le \frac{2\Delta}{\alpha T}+L\alpha\sigma^2
        +8Lb_2\alpha^2T\Delta+4Lb_2\alpha h_T\sigma^2.
 \label{eq:expect-before-absorb}
\end{align}
The inequality $L\alpha^3T^2\le h_T$ implies
\[
 8Lb_2\alpha^2T\Delta
 \le 8b_2h_T\frac{\Delta}{\alpha T}
 \le \frac{2\Delta}{\alpha T},\qquad
 4Lb_2\alpha h_T\sigma^2\le L\alpha\sigma^2.
\]
These two estimates prove~\eqref{eq:expect-bound}.

\emph{Expected maximum displacement.}
Define the deterministic energy quantity
\begin{equation}\label{eq:expected-energy-B}
 B_T=\alpha^2(TS_T+V_T).
\end{equation}
By~\eqref{eq:R-S-V}, $R_T\le2B_T$. Using~\eqref{eq:S-bound} and~\eqref{eq:V-R},
\begin{align}
 B_T
 &\le2\alpha T\Delta+\alpha^2(1+L\alpha T)V_T\notag\\
 &=2\alpha T\Delta+h_T V_T/T\notag\\
 &\le2\alpha T\Delta+h_T(\sigma^2+2b_2B_T).
 \label{eq:B-feedback}
\end{align}
Consequently, whenever $2b_2h_T<1$,
\begin{equation}\label{eq:B-resolvent}
 B_T\le\frac{2\alpha T\Delta+h_T\sigma^2}{1-2b_2h_T}.
\end{equation}

Let $M_k=\sum_{t=1}^k\xi_t$, $M_0=0$. This is a square-integrable vector martingale, and conditional Jensen makes $\norm{M_k}$ a nonnegative scalar submartingale. Doob's $L^2$ maximal inequality \citep{doob1953} and martingale orthogonality give
\begin{equation}\label{eq:doob-L2}
 \E\max_{0\le k\le T}\norm{M_k}^2
 \le4\E\norm{M_T}^2=4V_T.
\end{equation}
For completeness, put $U=\max_{0\le k\le T}\norm{M_k}$. Finite-horizon square integrability implies $U\in L^2$. Stopping the submartingale $\norm{M_k}$ at its first crossing of a level $\lambda>0$ gives
\[
 \lambda\Prob(U>\lambda)
 \le\E\{\norm{M_T}\1_{\{U>\lambda\}}\}.
\]
Integration in $\lambda$ and Cauchy--Schwarz yield
\[
 \E U^2\le2\E\{\norm{M_T}U\}
 \le2\{\E\norm{M_T}^2\}^{1/2}\{\E U^2\}^{1/2},
\]
which proves the factor $4$ in~\eqref{eq:doob-L2}. This scalar-norm argument has no dimension factor.

The exact SGD identity and the pathwise bound
\[
 \max_{k\le T}\norm{\sum_{t=1}^k g_t}
 \le\sum_{t=1}^T\norm{g_t}
 \le\left(T\sum_{t=1}^T\norm{g_t}^2\right)^{1/2}
\]
allow Minkowski's inequality in $L^2$ to be applied to the drift and martingale parts separately. Together with~\eqref{eq:doob-L2}, this gives
\begin{align}
 \mathcal R_T^{1/2}
 &\le\alpha\sqrt{TS_T}
       +\alpha\left\{\E\max_{k\le T}\norm{M_k}^2\right\}^{1/2}\notag\\
 &\le\alpha\sqrt{TS_T}+2\alpha\sqrt{V_T}.
 \label{eq:max-minkowski}
\end{align}
Since $(\sqrt u+2\sqrt v)^2\le5(u+v)$ for $u,v\ge0$,
\begin{equation}\label{eq:max-resolvent}
 \mathcal R_T\le5B_T
 \le\frac{10\alpha T\Delta+5h_T\sigma^2}{1-2b_2h_T}.
\end{equation}
Under~\eqref{eq:expect-condition}, the denominator is at least $1/2$, proving~\eqref{eq:expect-max-distance}. The marginal feedback is closed before Doob's inequality is applied; thus the maximal factor does not alter the stepsize restriction. Doob's inequality is applied to $M_k$, not to the SGD displacement, which also contains the predictable gradient drift.
\end{proof}

\begin{remark}[An unabsorbed version]\label{rem:resolvent}
Under Assumption~\ref{ass:objective}, (M2), and $0<\alpha\le L^{-1}$, the weaker restriction $2b_2h_T<1$ gives
\begin{align*}
 R_T&\le\frac{4\alpha T\Delta+2h_T\sigma^2}{1-2b_2h_T},\\
 \mathcal R_T&\le\frac{10\alpha T\Delta+5h_T\sigma^2}{1-2b_2h_T},\\
 \E A_T&\le\frac{2\Delta}{\alpha T}+L\alpha(\sigma^2+b_2R_T).
\end{align*}
The maximal bound is~\eqref{eq:max-resolvent}; its denominator is the same as that of the marginal bound. The constants in Theorem~\ref{thm:expectation} follow by imposing $b_2h_T\le1/4$ and are not optimized.
\end{remark}

\begin{corollary}[A one-sample $T^{-1/3}$ expectation rate]\label{cor:expect-rate}
Fix the objective and the constants in (M2). Choose ${a_0}>0$ such that
\begin{equation}\label{eq:expect-a}
 {a_0}\le L^{-1},\qquad b_2({a_0}^2+L{a_0}^3)\le\frac14,
\end{equation}
and set $\alpha={a_0}T^{-2/3}$. Then for every positive integer $T$,
\begin{equation}\label{eq:expect-rate}
 \E A_T\le\frac{4\Delta}{{a_0}}T^{-1/3}+2L{a_0}\sigma^2T^{-2/3}.
\end{equation}
At fixed $\delta\in(0,1)$ the right side of~\eqref{eq:expect-rate} divided by $\delta$ is a valid $1-\delta$ stationarity bound. Moreover,
\begin{equation}\label{eq:max-expect-rate}
 \E\max_{1\le t\le T+1}\norm{x_t-x_1}^2
 \le20a_0\Delta T^{1/3}
      +10\sigma^2\bigl(a_0^2T^{-1/3}+La_0^3\bigr).
\end{equation}
In particular, with fixed problem parameters,
\begin{equation}\label{eq:max-L2-rate}
 \left\|\max_{1\le t\le T+1}\norm{x_t-x_1}\right\|_2
 =O(T^{1/6}).
\end{equation}
\end{corollary}
\begin{proof}
Here $h_T={a_0}^2T^{-1/3}+L{a_0}^3\le {a_0}^2+L{a_0}^3$ and $\alpha\le {a_0}\le L^{-1}$. Apply Theorem~\ref{thm:expectation} and then Markov's inequality to the nonnegative variable $A_T$. Substituting $\alpha=a_0T^{-2/3}$ into~\eqref{eq:expect-max-distance} proves~\eqref{eq:max-expect-rate}; taking square roots gives~\eqref{eq:max-L2-rate}.
\end{proof}
A fully explicit admissible choice is any positive ${a_0}$ no larger than
\[
 \min\left\{L^{-1},(8b_2)^{-1/2},(8Lb_2)^{-1/3}\right\},
\]
with the two growth restrictions omitted if $b_2=0$. The absence of $b_2$ from the displayed noise term in~\eqref{eq:expect-bound} does not mean that noise growth is free: it enters the admissible stepsize through~\eqref{eq:expect-condition}.

\paragraph{Whole-run localization with second moments.}
For every $r>0$, Theorem~\ref{thm:expectation} and Markov's inequality give
\begin{equation}\label{eq:max-markov-localization}
 \Prob\!\left\{\max_{1\le t\le T+1}\norm{x_t-x_1}>r\right\}
 \le\min\!\left\{1,\frac{20\alpha T\Delta+10h_T\sigma^2}{r^2}\right\}.
\end{equation}
Thus a deterministic squared radius $(20\alpha T\Delta+10h_T\sigma^2)/\delta$ contains the whole run with probability at least $1-\delta$. If the numerator vanishes, the trajectory is constant almost surely. This second-moment localization has a polynomial Markov confidence cost; the finite-$p$ theorem below gives the separate variance and rare-shock dependence. The rate~\eqref{eq:max-L2-rate} controls the maximum within each fixed-horizon run, not a common bounded set for every horizon.

\Needspace{12\baselineskip}
\subsection{Oracle complexity and BG-0 optimality}\label{sec:minimax}
A power law with all parameters hidden does not identify the variance-growth cost. The next corollary makes the comparison quantitative and uses no higher moments.

\begin{corollary}[Parameter-explicit one-sample oracle complexity]\label{cor:complexity}
Assume Assumption~\ref{ass:objective} and (M2), let $\Delta>0$, and fix $\eps>0$. Choose an integer $T$ satisfying
\begin{equation}\label{eq:epsilon-horizon}
 T\ge\max\left\{1,\frac{8L\Delta}{\eps^2},
      \frac{32L\Delta\sigma^2}{\eps^4},
      \frac{512b_2\Delta^2}{\eps^4},
      \frac{4096Lb_2\Delta^3}{\eps^6}\right\},
\end{equation}
and set $\alpha=8\Delta/(T\eps^2)$.
Then $\E A_T\le\eps^2$. If $\tau$ is uniform on $\{1,\ldots,T\}$ and independent of the run, $\widehat x=x_\tau$ satisfies
\begin{equation}\label{eq:epsilon-output}
 \E\|\nabla F(\widehat x)\|^2\le\eps^2,
 \qquad \E\|\nabla F(\widehat x)\|\le\eps,
\end{equation}
using exactly $T$ stochastic-gradient queries. If $\eps^2\le L\Delta$, the sufficient budget has order
\begin{equation}\label{eq:optimal-complexity}
 T=O\left(1+\frac{L\Delta}{\eps^2}
              +\frac{L\Delta\sigma^2}{\eps^4}
              +\frac{Lb_2\Delta^3}{\eps^6}\right),
\end{equation}
with an absolute implied constant. Full-gradient evaluations are not needed to select the randomized output.
\end{corollary}
\begin{proof}
The smoothness restriction follows from $T\ge8L\Delta/\eps^2$. Direct substitution gives
\[
 b_2h_T=\frac{64b_2\Delta^2}{T\eps^4}
          +\frac{512Lb_2\Delta^3}{T\eps^6}\le\frac18+\frac18=\frac14.
\]
Theorem~\ref{thm:expectation} therefore applies. Its two stationarity terms satisfy
\[
 \frac{4\Delta}{\alpha T}=\frac{\eps^2}{2},\qquad
 2L\alpha\sigma^2=\frac{16L\Delta\sigma^2}{T\eps^2}
                       \le\frac{\eps^2}{2}.
\]
Conditionally on the iterates, $\E(\|\nabla F(x_\tau)\|^2\mid x_1,\ldots,x_T)=A_T$. Taking expectations and applying Cauchy--Schwarz proves~\eqref{eq:epsilon-output}. Finally, if $\eps^2\le L\Delta$, then
\[
 \frac{b_2\Delta^2}{\eps^4}
 \le\frac{Lb_2\Delta^3}{\eps^6}.
\]
Choose $T$ as the ceiling of the maximum in~\eqref{eq:epsilon-horizon}, and bound that maximum by a sum to obtain~\eqref{eq:optimal-complexity}.
\end{proof}

\paragraph{The lower-bound comparison.}
Consider the pointwise iid-oracle subclass in which $F$ is $L$-smooth, $F(x_1)-\inf F\le\Delta$, and
\begin{equation}\label{eq:bg-oracle-class}
 \E_Z\|G(x,Z)-\nabla F(x)\|^2
   \le\sigma^2+b_2\|x-x_1\|^2\qquad(x\in\R^d).
\end{equation}
The dimension is unrestricted. With the identification $B_v^2=b_2$, $b_v^2=\sigma^2$, and $x_0=x_1$, \citet[Theorem~1]{fazla2026} give the lower bound
\begin{equation}\label{eq:external-bg-lower}
 \Omega\left(\frac{Lb_2\Delta^3}{\eps^6}
                 +\frac{L\Delta\sigma^2}{\eps^4}\right)
\end{equation}
on the expected number of queries of any randomized algorithm achieving $\E\|\nabla F(\widehat x)\|\le\eps$. Their theorem assumes $L,\Delta>0$ and $\eps^2\le L\Delta/(1536\ell_1)$ for a numerical constant $\ell_1>0$. A lower bound for this first-moment criterion also applies to the stronger squared-gradient criterion in~\eqref{eq:epsilon-output}. Our deterministic query budget is, in particular, an expected query budget.

Thus~\eqref{eq:optimal-complexity} matches both stochastic terms of~\eqref{eq:external-bg-lower}, including their dependence on $L$, $b_2$, $\Delta$, and $\sigma^2$, with the usual deterministic $L\Delta/\eps^2$ overhead. In the small-accuracy stochastic regime where these terms dominate, ordinary single-sample SGD is minimax optimal in the smooth BG-0 class. In particular, for fixed $b_2>0$ and $L,\Delta>0$, its worst-case order is $\Theta(\eps^{-6})$. When $b_2=0$, the same corollary reduces to the bounded-variance $\eps^{-4}$ stochastic complexity. This comparison uses the cited lower-bound theorem; it is not a new oracle lower-bound proof.

\paragraph{Why the oracle class matters.}
The growth condition~\eqref{eq:bg-oracle-class} compares each query with the initialization. It imposes no common-sample two-point regularity. By contrast,
\[
 \|G(x,Z)-G(y,Z)\|_2\le\|G(x,Z)-G(y,Z)\|_p
                         \le L_{G,p}\|x-y\|
\]
shows that stochastic $L^p$-Lipschitz continuity with $p>2$ implies mean-square smoothness. Theorems~2 and~7 of \citet{fazla2026} distinguish that more structured oracle model and give an $\eps^{-4}$ worst-case exponent under their stated parameter restrictions, with variance-reduced upper bounds. The $\eps^{-6}$ BG-0 optimality claim therefore cannot be restricted automatically to that smaller class. Its best achievable rate for the unchanged single-sample SGD recursion remains a separate algorithm-specific question.

\subsection{A single run with decreasing stepsizes}\label{sec:single-run}
The horizon-tuned theorem gives the sharp finite-horizon stochastic complexity. We now keep one trajectory and change only the deterministic stepsize schedule. The reference point in the noise-growth condition remains the original initialization; it is not reset between prefixes.

\begin{theorem}[Weighted energy, maximal displacement, and almost-sure stationarity]\label{thm:single-run}
Let $F$ and $x_1$ satisfy the objective conditions of Assumption~\ref{ass:objective}. Consider
\begin{equation}\label{eq:variable-sgd}
 x_{t+1}=x_t-\eta_t(g_t+\xi_t),\qquad g_t=\nabla F(x_t),\qquad t\ge1,
\end{equation}
with deterministic $0<\eta_t\le L^{-1}$ and square-integrable martingale differences satisfying (M2) for all $t$. Define
\begin{equation}\label{eq:schedule-budgets}
 P_T=\sum_{t=1}^T\eta_t,\qquad
 Q_T=\sum_{t=1}^T\eta_t^2,\qquad
 C_T=\sum_{t=1}^T\eta_t^2P_{t-1},\qquad P_0=0.
\end{equation}
Suppose $Q_\infty,C_\infty<\infty$ and
\begin{equation}\label{eq:single-run-condition}
 \theta_\infty=2b_2(Q_\infty+LC_\infty)<1,
 \qquad
 V_*:=\frac{\sigma^2Q_\infty+4b_2\Delta C_\infty}{1-\theta_\infty}.
\end{equation}
Then
\begin{align}
 \sum_{t=1}^\infty\eta_t^2\E\|\xi_t\|^2&\le V_*,\notag\\
 \sum_{t=1}^\infty\eta_t\E\|g_t\|^2&\le2\Delta+LV_*.
 \label{eq:single-run-energies}
\end{align}
For every $T$,
\begin{equation}\label{eq:single-run-max}
 \E\max_{t\le T+1}\|x_t-x_1\|^2
 \le5\{P_T(2\Delta+LV_*)+V_*\}.
\end{equation}
If the stepsizes are nonincreasing, then
\begin{equation}\label{eq:single-run-average}
 \E A_T\le\frac{2\Delta+LV_*}{T\eta_T}.
\end{equation}
If, in addition, $P_\infty=\infty$, then
\begin{equation}\label{eq:single-run-as}
 \|\nabla F(x_t)\|\longrightarrow0\qquad\text{almost surely}.
\end{equation}
\end{theorem}
\begin{proof}
For this proof, use the weighted energies
\[
 U_T=\sum_{t=1}^T\eta_t\E\|g_t\|^2,
 \qquad V_T=\sum_{t=1}^T\eta_t^2\E\|\xi_t\|^2.
\]
The expected descent calculation with $\alpha$ replaced by $\eta_t$ gives
\begin{equation}\label{eq:variable-descent}
 U_T\le2\Delta+LV_T.
\end{equation}
For each $t$, the displacement identity, weighted Cauchy--Schwarz, and martingale orthogonality give
\begin{align}
 \E\|x_t-x_1\|^2
 &\le2P_{t-1}U_{t-1}+2V_{t-1}\notag\\
 &\le4P_{t-1}\Delta+2(1+LP_{t-1})V_{t-1}.
 \label{eq:variable-displacement}
\end{align}
Applying (M2), then using monotonicity of $V_T$,
\begin{align}
 V_T
 &\le\sigma^2Q_T+4b_2\Delta C_T
  +2b_2\sum_{t=1}^T\eta_t^2(1+LP_{t-1})V_{t-1}\notag\\
 &\le\sigma^2Q_\infty+4b_2\Delta C_\infty+\theta_\infty V_T.
 \label{eq:variable-feedback}
\end{align}
Absorb the last term and let $T\to\infty$ to prove the first assertion of~\eqref{eq:single-run-energies}. Equation~\eqref{eq:variable-descent} proves the second. If $\eta_t$ is nonincreasing, $T\eta_T\E A_T\le U_T$, proving~\eqref{eq:single-run-average}.

For the maximal displacement, the martingale $N_k=\sum_{t=1}^k\eta_t\xi_t$ satisfies
$\E\max_{k\le T}\|N_k\|^2\le4V_T$ by Doob. The deterministic drift satisfies
\[
 \max_{k\le T}\Big\|\sum_{t=1}^k\eta_tg_t\Big\|
 \le\left(P_T\sum_{t=1}^T\eta_t\|g_t\|^2\right)^{1/2}.
\]
Minkowski therefore bounds the square root of the maximal-displacement expectation by
$\sqrt{P_TU_T}+2\sqrt{V_T}$. Squaring and using
$(\sqrt u+2\sqrt v)^2\le5(u+v)$ proves~\eqref{eq:single-run-max}.

It remains to prove~\eqref{eq:single-run-as}, rather than only a liminf statement. By~\eqref{eq:single-run-energies} and monotone convergence,
\begin{equation}\label{eq:as-energy-tail}
 \sum_{t=1}^\infty\eta_t\|g_t\|^2<\infty\quad\text{almost surely}.
\end{equation}
The martingale $N_k$ is bounded in $L^2$ and thus converges almost surely and in $L^2$ \citep{doob1953}. Work on the probability-one event on which both assertions hold. Set
\[
 E_m=\sum_{t=m}^\infty\eta_t\|g_t\|^2\longrightarrow0,
 \qquad d_m=\sup_{k\ge m-1}\|N_k-N_{m-1}\|\longrightarrow0.
\]
For any $k\ge m$ with $\sum_{t=m}^{k-1}\eta_t\le1+L^{-1}$,
\begin{equation}\label{eq:as-local-displacement}
 \|x_k-x_m\|\le\sqrt{(1+L^{-1})E_m}+d_m.
\end{equation}
Here the drift is bounded by weighted Cauchy--Schwarz, and the other term is the martingale tail. Given $m$, let $n(m)$ be the first index for which
$\sum_{t=m}^{n(m)}\eta_t\ge1$; it exists because $P_\infty=\infty$.
For all $m\le k\le n(m)$,~\eqref{eq:as-local-displacement} and $L$-smoothness show that $\|g_k-g_m\|$ tends to zero uniformly over this interval as $m\to\infty$. If $\|g_m\|\ge\eps$ for arbitrarily large $m$, all these gradients would eventually have norm at least $\eps/2$, implying
\[
 E_m\ge\sum_{t=m}^{n(m)}\eta_t\|g_t\|^2\ge\eps^2/4,
\]
a contradiction. Thus $\|g_m\|\to0$.
\end{proof}

\begin{corollary}[A horizon-free schedule arbitrarily close to the $T^{-1/3}$ rate]\label{cor:horizon-free}
Fix $\kappa>1/3$ and set
\begin{equation}\label{eq:horizon-free-step}
 \eta_t=a_0t^{-2/3}\{\log(et)\}^{-\kappa}.
\end{equation}
For sufficiently small fixed $a_0>0$, the assumptions on the schedule in Theorem~\ref{thm:single-run} hold, and the same run satisfies
\begin{equation}\label{eq:horizon-free-rate}
 \E A_T=O\!\left(T^{-1/3}(\log T)^\kappa\right),\qquad
 \|\nabla F(x_t)\|\to0\quad\text{almost surely}.
\end{equation}
Also, its maximal displacement has squared expectation $O(T^{1/3}(\log T)^{-\kappa})$ as $T\to\infty$.
\end{corollary}
\begin{proof}
The schedule is nonincreasing and $\eta_1=a_0$. An integral comparison, or differentiation of $t^{1/3}(\log t)^{-\kappa}$, gives
$P_T\asymp T^{1/3}(\log T)^{-\kappa}$ and $P_\infty=\infty$. Clearly $Q_\infty<\infty$. Moreover,
\[
 \eta_t^2P_{t-1}=O\!\left(t^{-1}(\log t)^{-3\kappa}\right),
\]
so $C_\infty<\infty$ when $3\kappa>1$. The two budgets scale as $Q_\infty=O(a_0^2)$ and $C_\infty=O(a_0^3)$, with constants depending only on $\kappa$. Choosing $a_0$ small gives~\eqref{eq:single-run-condition} and $a_0\le L^{-1}$. The conclusions now follow from~\eqref{eq:single-run-average}, \eqref{eq:single-run-as}, and~\eqref{eq:single-run-max}.
\end{proof}
The last-iterate statement here is qualitative almost-sure stationarity. The finite-time rate concerns $A_T$; no finite-time last-iterate rate or convergence to a particular point is asserted.

\section{The probability input and predictable stopping}\label{sec:probability}
The finite-$p$ concentration input is the following maximal inequality. The second-moment Doob inequality was proved directly in the proof of Theorem~\ref{thm:expectation}. The maximal form and its constants are already supplied by \citet[Theorem~2.2]{mollenhauer2025}; no additional maximal-inequality result is claimed here. Hilbert spaces have smoothness constant $D=1$, independently of dimension.\footnote{Equation~(2.2) of the cited version has the index slip $\max_i\|M_n\|$. We use the intended $\max_i\|M_i\|$, consistent with its maximal formulation and the independent-sum specialization (2.5).}

\begin{theorem}[Maximal Hilbert-space Fuk--Nagaev inequality]\label{thm:FN}
\textnormal{\citep[Theorem~2.2, Hilbert-space specialization]{mollenhauer2025}.}
Let $p>2$ and let $X_1,\ldots,X_T$ be martingale differences in a separable Hilbert space. If deterministic $v,m_p\ge0$ satisfy
\[
 \sum_{t=1}^T\E_{t-1}\norm{X_t}^2\le v,\qquad
 \sum_{t=1}^T\E_{t-1}\norm{X_t}^p\le m_p
 \quad\hbox{almost surely},
\]
then, for every $u\in(0,1)$, with probability at least $1-u$,
\begin{equation}\label{eq:FN}
 \max_{1\le k\le T}\norm{\sum_{t=1}^k X_t}
 \le \sqrt{2v\log(2/u)}+\cprob(2m_p/u)^{1/p},
\end{equation}
Here
\[
 \cprob=\frac{1}{2p}+\min\{1/p,1/5\}+1+\1_{\{p>3\}}\frac p3.
\]
\end{theorem}
The terminal bound follows by taking $k=T$. The budget $m_p$ is a sum of conditional $p$th moments, not a squared moment scale or a stepsize coefficient.

\begin{lemma}[Predictable stopping preserves the budgets]\label{lem:maximal}
Under the assumptions of Theorem~\ref{thm:FN}, let $I_t\in\{0,1\}$ be $\F_{t-1}$-measurable. The stopped increments $X_t^\circ=I_tX_t$ satisfy the same maximal bound~\eqref{eq:FN} with budgets $v,m_p$.
\end{lemma}
\begin{proof}
Conditional centering gives $\E_{t-1}X_t^\circ=0$, and for $q=2,p$,
\[
 \E_{t-1}\norm{X_t^\circ}^q=I_t\E_{t-1}\norm{X_t}^q
                            \le\E_{t-1}\norm{X_t}^q.
\]
Sum and apply Theorem~\ref{thm:FN}. For a stopping time $\tau$, the convention $I_t=\1_{\{\tau\ge t\}}$ is predictable and retains the exit-producing increment. No conditioning on a future no-exit event is involved.
\end{proof}

\begin{lemma}[Raw energy as a Hilbert-space martingale]\label{lem:energy}
Suppose $\xi_t\in\R^d$ are martingale differences with deterministic conditional envelopes
\begin{equation}\label{eq:uniform-moments}
 \E_{t-1}\norm{\xi_t}^2\le s^2,\qquad
 \E_{t-1}\norm{\xi_t}^p\le n_p^p.
\end{equation}
Then with probability at least $1-u$,
\begin{equation}\label{eq:raw-energy}
 \left(\sum_{t=1}^T\norm{\xi_t}^2\right)^{1/2}
 \le \sqrt{2Ts^2\log(2/u)}+\cprob n_p(2T/u)^{1/p}.
\end{equation}
The same expression bounds the maximum norm of the noise partial sums with probability at least $1-u$.
\end{lemma}
\begin{proof}
In the Hilbert space $\mathcal H_T=(\R^d)^T$, let $E_t$ insert a vector into the $t$th coordinate and set $X_t=E_t\xi_t$. The map $E_t$ is deterministic and isometric, so $(X_t)$ are Hilbert-valued martingale differences with second- and $p$th-moment budgets $Ts^2$ and $Tn_p^p$. Moreover,
\[
 \norm{\sum_{t=1}^T E_t\xi_t}_{\mathcal H_T}^2
 =\sum_{t=1}^T\norm{\xi_t}^2.
\]
Theorem~\ref{thm:FN} proves~\eqref{eq:raw-energy}. Apply Theorem~\ref{thm:FN} in $\R^d$ to obtain the partial-sum assertion. These are two separate probability statements; simultaneous use requires a union bound.
\end{proof}

\begin{lemma}[Centered noise energy without a logarithmic variance baseline]\label{lem:centered-energy}
Suppose the deterministic conditional envelopes~\eqref{eq:uniform-moments} hold, with $p>2$. For every $u\in(0,1)$,
\begin{equation}\label{eq:centered-energy}
 \Prob\!\left\{\sum_{t=1}^T\|\xi_t\|^2>
 \eprob\left[Ts^2+n_p^2(T/u)^{2/p}\right]\right\}\le u,
\end{equation}
where one explicit choice is
\begin{equation}\label{eq:energy-constant}
 \eprob=\begin{cases}
 2^{1+2/p},&2<p\le4,\\
 1+2^{1/(p-2)}\sqrt{(p-2)/e}
      +2^{1+2/p}\mathfrak c_{p/2},&p>4.
 \end{cases}
\end{equation}
Here $\mathfrak c_{p/2}$ is the scalar constant in Theorem~\ref{thm:FN}, at moment order $p/2>2$.
\end{lemma}
\begin{proof}
Write $q=p/2$, $Y_t=\|\xi_t\|^2$, and $D_t=Y_t-\E_{t-1}Y_t$.
Then $D_t$ are scalar martingale differences, $\sum_t\E_{t-1}Y_t\le Ts^2$, and conditional Jensen gives
\begin{equation}\label{eq:centered-q-moment}
 \E_{t-1}|D_t|^q\le2^q n_p^p.
\end{equation}
For $1<q\le2$, the elementary inequality
\[
 |a+b|^q\le|a|^q+q\operatorname{sgn}(a)|a|^{q-1}b+2|b|^q
\]
(with $\operatorname{sgn}(0)=0$) follows by integrating the $(q-1)$-H\"older derivative of $|a|^q$. Conditioning and iteration give
\[
 \E\Big|\sum_tD_t\Big|^q\le2\sum_t\E|D_t|^q
 \le2^{q+1}Tn_p^p.
\]
Markov's inequality yields a deviation allowance
$2^{1+2/p}n_p^2(T/u)^{2/p}$, proving~\eqref{eq:centered-energy} for $p\le4$.

For $p>4$, conditional interpolation gives
\begin{equation}\label{eq:fourth-interpolation}
 \E_{t-1}D_t^2\le\E_{t-1}\|\xi_t\|^4
 \le s^{2(p-4)/(p-2)}n_p^{2p/(p-2)}.
\end{equation}
Apply Theorem~\ref{thm:FN} in $\R$ with moment order $q$ to the $D_t$. Its polynomial term is
$2^{1+2/p}\mathfrak c_{p/2}n_p^2(T/u)^{2/p}$.
For its Gaussian term, put
\[
 A=Ts^2,\qquad B=n_p^2(T/u)^{2/p},\qquad
 \beta_p=\frac{p-4}{2(p-2)}\in(0,1).
\]
A direct rearrangement gives
\begin{align*}
 &\sqrt{2T\log(2/u)}\,
 s^{(p-4)/(p-2)}n_p^{p/(p-2)}\\
 &\hspace{1em}=\sqrt2\,u^{1/(p-2)}\sqrt{\log(2/u)}\,
 A^{\beta_p}B^{1-\beta_p}\\
 &\hspace{1em}\le2^{1/(p-2)}\sqrt{(p-2)/e}\,
       \{\beta_p A+(1-\beta_p)B\}.
\end{align*}
The last inequality uses weighted arithmetic--geometric means and
$\sup_{0<u<1}\sqrt2\,u^{1/(p-2)}\sqrt{\log(2/u)}
=2^{1/(p-2)}\sqrt{(p-2)/e}$ for $p>4$.
Adding the predictable mean $A$ proves the stated constant. If an envelope vanishes, the zero-noise case follows directly; the same inequalities otherwise follow by continuity.
\end{proof}
The lemma controls an upper deviation of a nonnegative energy. It is not a centered two-sided Gaussian-plus-polynomial assertion at moment order $p/2\le2$. The Gaussian term available for $p>4$ is absorbed into the mean and rare-shock allowances; no logarithm multiplies $Ts^2$.

\section{A deterministic-envelope lemma for plain SGD}\label{sec:uniform}
This section proves the deterministic-envelope optimization ingredient directly from Theorem~\ref{thm:FN}. The raw-energy lift and objective-level localization are auxiliary ingredients, not new concentration theorems. Working directly with plain SGD permits the smoothness cap $\alpha\le L^{-1}$. The relation to the supplied unpublished manuscript is recorded in Section~\ref{sec:provenance}.

For $p>2$, define the numerical constant
\begin{equation}\label{eq:Dp}
 \Dprob=2+\cprob^2 6^{2/p}.
\end{equation}
For a horizon and failure probability, use
\begin{equation}\label{eq:ell-rho}
 \ell=\log(6/\delta),\qquad \rho=T^{2/p-1}\delta^{-2/p}.
\end{equation}

\begin{proposition}[Uniform envelopes: simultaneous trajectory bounds]\label{prop:uniform}
Under Assumption~\ref{ass:objective}, suppose~\eqref{eq:uniform-moments} holds and $0<\alpha\le L^{-1}$. Put
\begin{equation}\label{eq:J}
 J=\sqrt{2Ts^2\ell}+\cprob n_p(6T/\delta)^{1/p},
 \qquad H=4\Delta+6L\alpha^2J^2.
\end{equation}
With probability at least $1-\delta$, the following four inequalities hold simultaneously:
\begin{align}
 \sum_{t=1}^T\norm{\xi_t}^2&\le J^2,
 &\max_{k\le T}\norm{\sum_{t=1}^k\xi_t}&\le J,
 \label{eq:uniform-noise}\\
 \max_{1\le t\le T+1}\{F(x_t)-F_*\}&\le H,
 &A_T&\le\frac{4\Delta}{\alpha T}+\frac{12L\alpha J^2}{T}.
 \label{eq:uniform-objective}
\end{align}
Also,
\begin{equation}\label{eq:J-square}
 J^2\le\Dprob T(s^2\ell+n_p^2\rho).
\end{equation}
\end{proposition}
\begin{proof}
\emph{Step 1: a pathwise descent inequality.}
For arbitrary vectors $\xi_t$, expanding the descent lemma gives
\begin{align*}
 F(x_{t+1})
 &\le F(x_t)-\alpha\norm{g_t}^2-\alpha\ip{g_t}{\xi_t}
       +\frac{L\alpha^2}{2}\norm{g_t+\xi_t}^2\\
 &=F(x_t)-\alpha(1-L\alpha/2)\norm{g_t}^2
       -\alpha(1-L\alpha)\ip{g_t}{\xi_t}
       +\frac{L\alpha^2}{2}\norm{\xi_t}^2.
\end{align*}
As $\alpha\le L^{-1}$, for each $k\le T$,
\begin{equation}\label{eq:pathwise-descent}
 F(x_{k+1})-F_*+\frac\alpha2\sum_{t=1}^k\norm{g_t}^2
 \le\Delta+\alpha M_k+\frac{L\alpha^2}{2}\sum_{t=1}^k\norm{\xi_t}^2,
 \quad
 M_k=-(1-L\alpha)\sum_{t=1}^k\ip{g_t}{\xi_t}.
\end{equation}

\emph{Step 2: three probability events.}
Apply Lemma~\ref{lem:energy} to the raw energy with $u=\delta/3$, and Theorem~\ref{thm:FN} to the noise partial sums with the same $u$. The two inequalities in~\eqref{eq:uniform-noise} fail with total probability at most $2\delta/3$.

For $H>0$, set
\[
 I_t=\1_{\{\max_{1\le s\le t}(F(x_s)-F_*)\le H\}},\qquad
 M_k^\circ=-(1-L\alpha)\sum_{t=1}^k I_t\ip{g_t}{\xi_t}.
\]
These indicators are predictable. Lemma~\ref{lem:self-bound} implies
$\norm{(1-L\alpha)I_tg_t}\le\sqrt{2LH}$. The scalar increments of $M^\circ$ therefore have conditional moment totals bounded by $2LH\,Ts^2$ and $(2LH)^{p/2}Tn_p^p$. Theorem~\ref{thm:FN}, applied to the predictable stopped transform with $u=\delta/3$, gives
\begin{equation}\label{eq:scalar-event}
 \max_{k\le T} M_k^\circ\le\sqrt{2LH}\,J
\end{equation}
outside an event of probability at most $\delta/3$. Intersect the three events. No concentration inequality has been conditioned on an event involving future observations.

\emph{Step 3: removing the objective stop.}
Initially $F(x_1)-F_*=\Delta\le H$. If the first exit were at $x_{s+1}$, $s+1\le T+1$, then $M_s=M_s^\circ$, including the increment producing that exit. Dropping the nonnegative gradient term in~\eqref{eq:pathwise-descent},
\begin{align*}
 F(x_{s+1})-F_*
 &\le\Delta+\alpha\sqrt{2LH}\,J+\frac{L\alpha^2J^2}{2}\\
 &\le\frac H2+\Delta+\frac32L\alpha^2J^2=\frac{3H}{4}.
\end{align*}
The second inequality uses $\alpha\sqrt{2LH}J\le H/2+L\alpha^2J^2$. This contradicts the exit. Hence the objective remains below $H$ throughout, and~\eqref{eq:scalar-event} applies to the unstopped transform on the good event.

\emph{Step 4: stationarity.}
At $k=T$, retaining the gradient sum in~\eqref{eq:pathwise-descent} gives
\[
 A_T\le\frac{2\Delta}{\alpha T}+\frac{2\sqrt{2LH}\,J}{T}
                                  +\frac{L\alpha J^2}{T}.
\]
Set $d_0=\Delta/(\alpha T)$ and $e_0=L\alpha J^2/T$. Since $H=4\Delta+6L\alpha^2J^2$, the middle term equals $2\sqrt{8d_0e_0+12e_0^2}$, which is bounded by
\[
 4\sqrt2\sqrt{d_0e_0}+4\sqrt3\,e_0
 \le2d_0+(4+4\sqrt3)e_0.
\]
As $5+4\sqrt3<12$, this proves~\eqref{eq:uniform-objective}.
Finally, write
\[
 J/\sqrt T=\sqrt2\,s\sqrt\ell+\cprob6^{1/p}n_p\sqrt\rho.
\]
Cauchy--Schwarz gives
$J^2\le(2+\cprob^2 6^{2/p})T(s^2\ell+n_p^2\rho)$,
which is~\eqref{eq:J-square}. This is one half of the constant obtained by separately doubling both squared terms.

If $H=0$, then $\Delta=0$ and $J=0$. The conditional second-moment envelope vanishes, so all errors vanish almost surely, and Lemma~\ref{lem:self-bound} makes the initial gradient zero. The recursion is constant and all conclusions hold directly.
\end{proof}

\subsection{Stopping objective plus accumulated gradient energy}\label{sec:energy-stop}
The baseline proposition bounds the adaptive transform through a maximum gradient norm. Retaining accumulated gradient energy in the stopping variable gives a smaller variance budget for that transform. To make the constants in the refinements unambiguous, define
\begin{equation}\label{eq:refined-constants}
 \Hprob=\max\{32,\ 2\,3^{2/p}\eprob+8\cprob^2 12^{2/p}\},
 \qquad \Kprob=4\Hprob+2\Dprob.
\end{equation}
These constants depend only on $p$. They are conservative; the purpose of the refinement is the placement of the confidence factors rather than an optimized numerical prefactor.

\begin{proposition}[Energy-stopped deterministic-envelope bounds]\label{prop:sharp-uniform}
Under the assumptions of Proposition~\ref{prop:uniform}, define
\begin{align}
 E&=\eprob\{Ts^2+n_p^2(3T/\delta)^{2/p}\},\notag\\
 U&=L^{1-2/p}\alpha^{2-2/p}n_p^2\delta^{-2/p},\notag\\
 H&=4\left\{\Delta+\frac{L\alpha^2}{2}E
            +8\alpha s^2\ell+2\cprob^2 12^{2/p}U\right\}.
 \label{eq:sharp-envelope-H}
\end{align}
With probability at least $1-\delta$, simultaneously,
\begin{align}
 \sum_{t=1}^T\|\xi_t\|^2&\le E,\qquad
 \max_{k\le T}\Big\|\sum_{t=1}^k\xi_t\Big\|\le J,\notag\\
 \max_{t\le T+1}\mathcal Y_t&\le H,\qquad
 \mathcal Y_t:=F(x_t)-F_*+\frac\alpha2\sum_{s<t}\|g_s\|^2.
 \label{eq:sharp-envelope-event}
\end{align}
Here $J$ is~\eqref{eq:J}. If also $L\alpha T\ge1$, put
\begin{equation}\label{eq:Phi-envelope}
 \Phi(s,n_p)=\frac{\Delta}{\alpha T}
        +L\alpha(s^2+n_p^2\rho)+\frac{s^2\ell}{T}.
\end{equation}
On that same event,
\begin{align}
 A_T&\le2\Hprob\Phi(s,n_p),\notag\\
 \max_{t\le T+1}\{F(x_t)-F_*\}&\le\Hprob\alpha T\Phi(s,n_p),\notag\\
 \max_{t\le T+1}\|x_t-x_1\|^2&\le\Kprob\alpha^2T^2\Phi(s,n_p).
 \label{eq:sharp-envelope-conclusions}
\end{align}
The explicit $H$ in~\eqref{eq:sharp-envelope-H} and event~\eqref{eq:sharp-envelope-event} are valid without $L\alpha T\ge1$.
\end{proposition}
\begin{proof}
Lemma~\ref{lem:centered-energy} with $u=\delta/3$ gives the energy event. Theorem~\ref{thm:FN} with $u=\delta/3$ gives the partial-sum event with bound $J$.
Assume $H>0$ and set
\[
 I_t=\1_{\{\max_{s\le t}\mathcal Y_s\le H\}},\qquad
 M_k^\circ=-(1-L\alpha)\sum_{t=1}^k I_t\langle g_t,\xi_t\rangle.
\]
The indicators are predictable and retain the possible exit-producing update. The retained indices form a prefix. If $j$ is the last retained index, then
\[
 \sum_{t<j}\|g_t\|^2\le2H/\alpha,
 \qquad \|g_j\|^2\le2LH,
\]
by $\mathcal Y_j\le H$ and Lemma~\ref{lem:self-bound}. The same reasoning applies if every index is retained. Since $\alpha\le L^{-1}$, on every path
\begin{equation}\label{eq:stopped-gradient-budgets}
 \max_t I_t\|g_t\|^2\le2LH,
 \qquad \sum_{t=1}^T I_t\|g_t\|^2\le4H/\alpha.
\end{equation}
Consequently the transform increments have deterministic conditional budgets
\begin{equation}\label{eq:sharp-transform-budgets}
 v_M\le4Hs^2/\alpha,\qquad
 m_{p,M}\le(4H/\alpha)(2LH)^{p/2-1}n_p^p.
\end{equation}
Applying Theorem~\ref{thm:FN} to the stopped transform with $u=\delta/3$ gives
\begin{align*}
 \alpha\max_k M_k^\circ
 &\le\sqrt H\left\{\sqrt{8\alpha s^2\ell}
 +\sqrt2\,\cprob12^{1/p}L^{1/2-1/p}\alpha^{1-1/p}
 n_p\delta^{-1/p}\right\}\\
 &\le H/2+8\alpha s^2\ell+2\cprob^2 12^{2/p}U.
\end{align*}
The second line uses $\sqrt H B\le H/2+B^2/2$ and $(a+b)^2\le2a^2+2b^2$. Intersect the three events, whose total failure probability is at most $\delta$.

The pathwise inequality~\eqref{eq:pathwise-descent} reads
\[
 \mathcal Y_{k+1}\le\Delta+\alpha M_k
                       +\frac{L\alpha^2}{2}\sum_{t\le k}\|\xi_t\|^2.
\]
At a hypothetical first exit above $H$, the original and stopped transforms agree through the exit-producing increment. The displayed bounds and~\eqref{eq:sharp-envelope-H} would give $\mathcal Y_{k+1}\le3H/4$, a contradiction. This proves~\eqref{eq:sharp-envelope-event} and implies $A_T\le2H/(\alpha T)$.

When $L\alpha T\ge1$, the extra term satisfies
\begin{equation}\label{eq:U-absorption}
 U\le L\alpha^2T n_p^2\rho,\qquad
 \frac{U}{L\alpha^2T n_p^2\rho}=(L\alpha T)^{-2/p}
\end{equation}
whenever the ratio is defined; if $n_p=0$ both terms vanish. Moreover,
$E\le3^{2/p}\eprob T(s^2+n_p^2\rho)$.
The definition~\eqref{eq:refined-constants} therefore gives
$H\le\Hprob\alpha T\Phi(s,n_p)$, proving the first two conclusions.
For displacement, use the exact decomposition as in~\eqref{eq:radius-feedback}:
\[
 \max_t\|x_t-x_1\|^2
 \le2\alpha^2T^2 A_T+2\alpha^2J^2.
\]
The first term is at most $4\Hprob\alpha^2T^2\Phi$. By~\eqref{eq:J-square}, the second is at most
$2\Dprob\alpha^2T(s^2\ell+n_p^2\rho)\le2\Dprob\alpha^2T^2\Phi$,
using $L\alpha T\ge1$. This proves the third conclusion. If $H=0$, the initial gradient and all errors vanish, so the result is immediate.
\end{proof}

\clearpage
\section{Nagaev bounds under distance-dependent moments}\label{sec:main}
We now replace the deterministic envelopes by Assumption~\ref{ass:p}. The radius in the next theorem is part of its conclusion; it is not a condition imposed on the optimizer.

\begin{theorem}[Distance-dependent Nagaev bound]\label{thm:main}
Suppose Assumptions~\ref{ass:objective} and~\ref{ass:p} hold. For $\delta\in(0,1)$ define $\ell,\rho$ by~\eqref{eq:ell-rho}, $h_T$ by~\eqref{eq:h}, and
\begin{equation}\label{eq:W}
 W_0=\sigma^2\ell+\nu_p^2\rho,\qquad
 W_1=b_2\ell+b_p\rho.
\end{equation}
Let $\Dprob$ be the explicit constant in~\eqref{eq:Dp}. If
\begin{equation}\label{eq:main-condition}
 \boxed{\qquad 0<\alpha\le L^{-1},\qquad
 h_TW_1\le\frac{1}{96\Dprob},\qquad}
\end{equation}
then, with probability at least $1-\delta$,
\begin{equation}\label{eq:main-bound}
 \boxed{\quad A_T\le\frac{6\Delta}{\alpha T}
       +18\Dprob L\alpha\bigl\{\sigma^2\log(6/\delta)
                    +\nu_p^2T^{2/p-1}\delta^{-2/p}\bigr\}.\quad}
\end{equation}
On the same event,
\begin{align}
 \max_{1\le t\le T+1}\norm{x_t-x_1}^2
 &\le R^2:=16\alpha T\Delta+48\Dprob h_TW_0,
 \label{eq:main-radius}\\
 \max_{1\le t\le T+1}\{F(x_t)-F_*\}
 &\le5\Delta+9\Dprob L\alpha^2TW_0.
 \label{eq:main-objective}
\end{align}
All numerical constants are independent of $d$; dependence on dimension may enter through the moment parameters.
\end{theorem}

\subsection{Proof of Theorem~\ref{thm:main}}
\begin{proof}
\emph{Step 1: deterministic localization and an auxiliary recursion.}
First assume $R>0$, where $R$ is given by~\eqref{eq:main-radius}. On the probability space of the original recursion, define
\begin{equation}\label{eq:distance-stop}
 I_t=\1_{\{\max_{1\le s\le t}\norm{x_s-x_1}\le R\}},
 \qquad \widetilde\xi_t=I_t\xi_t.
\end{equation}
The indicators are $\F_{t-1}$-measurable. As in Lemma~\ref{lem:maximal}, $\widetilde\xi_t$ are martingale differences, and the growth bounds imply
\begin{align}
 \E_{t-1}\norm{\widetilde\xi_t}^2
 &\le\sigma_R^2:=\sigma^2+b_2R^2,\notag\\
 \E_{t-1}\norm{\widetilde\xi_t}^p
 &\le\nu_R^p,\qquad \nu_R^2:=\nu_p^2+b_pR^2.
 \label{eq:stopped-moment-envelopes}
\end{align}
Indeed, on $\{I_t=1\}$ the original iterate is within the radius, and on $\{I_t=0\}$ the left sides vanish.

Define the auxiliary process by
\begin{equation}\label{eq:auxiliary}
 \widetilde x_1=x_1,\qquad
 \widetilde x_{t+1}=\widetilde x_t
      -\alpha\{\nabla F(\widetilde x_t)+\widetilde\xi_t\}.
\end{equation}
This is an adapted recursion driven by the stopped errors on the original filtration. It is not an algorithm to be run, and no assertion that $\widetilde\xi_t=e(\widetilde x_t,Z_t)$ is needed. Proposition~\ref{prop:uniform} requires only the martingale property and the deterministic moment envelopes, both of which hold in~\eqref{eq:stopped-moment-envelopes}.

Let $\widetilde A_T=T^{-1}\sum_{t=1}^T\norm{\nabla F(\widetilde x_t)}^2$, and let $J_R$ be~\eqref{eq:J} with $s=\sigma_R$ and $n_p=\nu_R$. Proposition~\ref{prop:uniform} yields an event $\mathcal E_R$ with $\Prob(\mathcal E_R)\ge1-\delta$ on which all its trajectory bounds hold for the auxiliary process. Moreover,
\begin{equation}\label{eq:JR-square}
 J_R^2\le\Dprob T(W_0+W_1R^2).
\end{equation}
The radius and both noise envelopes are deterministic before applying the concentration inequality.

\emph{Step 2: the auxiliary process cannot leave the radius.}
On $\mathcal E_R$, the displacement identity for~\eqref{eq:auxiliary}, Cauchy--Schwarz, and the maximal-noise bound give
\begin{align}
 \max_{k\le T}\norm{\widetilde x_{k+1}-x_1}^2
 &\le2\alpha^2T^2\widetilde A_T+2\alpha^2J_R^2\notag\\
 &\le8\alpha T\Delta+(24L\alpha^3T+2\alpha^2)J_R^2\notag\\
 &\le8\alpha T\Delta+24\Dprob h_T(W_0+W_1R^2).
 \label{eq:radius-feedback}
\end{align}
The last step uses~\eqref{eq:JR-square} and
$24L\alpha^3T^2+2\alpha^2T\le24h_T$.
By the definition of $R^2$, the sum of the first term and $24\Dprob h_TW_0$ is $R^2/2$. The remaining term is at most $R^2/4$ by~\eqref{eq:main-condition}. Therefore
\begin{equation}\label{eq:strict-radius}
 \max_{k\le T}\norm{\widetilde x_{k+1}-x_1}^2\le\frac34R^2
 \qquad\hbox{on }\mathcal E_R.
\end{equation}

\emph{Step 3: equality with the original recursion.}
Suppose that on $\mathcal E_R$ the original process first exits the closed radius-$R$ ball at time $\tau\le T+1$. For every $t<\tau$, $I_t=1$, so $\widetilde\xi_t=\xi_t$. An induction using the identical initializations in~\eqref{eq:sgd} and~\eqref{eq:auxiliary} gives
$\widetilde x_t=x_t$ for every $t\le\tau$. The increment at $t=\tau-1$ is included in this equality. But then
$\norm{\widetilde x_\tau-x_1}=\norm{x_\tau-x_1}>R$, contradicting~\eqref{eq:strict-radius}. Hence no original exit occurs, the two recursions agree up to time $T+1$, and~\eqref{eq:main-radius} follows. This is a pathwise argument on $\mathcal E_R$, not conditioning a martingale inequality on a future no-exit event.

\emph{Step 4: removal of the radius from the stationarity bound.}
On $\mathcal E_R$, Proposition~\ref{prop:uniform} and~\eqref{eq:JR-square} imply
\begin{align}
 A_T
 &\le\frac{4\Delta}{\alpha T}+12\Dprob L\alpha(W_0+W_1R^2)\notag\\
 &=\frac{4\Delta}{\alpha T}+12\Dprob L\alpha W_0
       +192\Dprob L\alpha^2TW_1\Delta
       +576\Dprob^2L\alpha h_TW_1W_0.
 \label{eq:main-before-absorb}
\end{align}
Since $L\alpha^3T^2\le h_T$,
\[
 192\Dprob L\alpha^2TW_1\Delta
 \le192\Dprob h_TW_1\frac{\Delta}{\alpha T}
 \le\frac{2\Delta}{\alpha T}.
\]
Similarly,
\[
 576\Dprob^2L\alpha h_TW_1W_0\le6\Dprob L\alpha W_0.
\]
Substitution in~\eqref{eq:main-before-absorb} proves~\eqref{eq:main-bound}.

For the objective maximum, Proposition~\ref{prop:uniform} gives
\begin{align*}
 \max_t\{F(x_t)-F_*\}
 &\le4\Delta+6\Dprob L\alpha^2T(W_0+W_1R^2)\\
 &=4\Delta+6\Dprob L\alpha^2TW_0
       +96\Dprob L\alpha^3T^2W_1\Delta
       +288\Dprob^2L\alpha^2Th_TW_1W_0\\
 &\le5\Delta+9\Dprob L\alpha^2TW_0,
\end{align*}
which is~\eqref{eq:main-objective}.

If $R=0$, then $\Delta=0$ and $W_0=0$. Since $\ell,\rho>0$, the moment intercepts vanish. Lemma~\ref{lem:self-bound} gives $g_1=0$, and the moment bounds at $x_1$ give $\xi_1=0$ almost surely. Inductively the original process remains at $x_1$ and all errors vanish. The conclusions follow directly.
\end{proof}

\subsection{A sharper variance baseline and feedback condition}
\begin{theorem}[Energy-refined distance-dependent Nagaev bound]\label{thm:sharp-main}
Suppose Assumptions~\ref{ass:objective} and~\ref{ass:p} hold. Define
\begin{align}
 \Phi_0&=\frac{\Delta}{\alpha T}
       +L\alpha(\sigma^2+\nu_p^2\rho)+\frac{\sigma^2\ell}{T},\notag\\
 \Xi_T&=\alpha^2Tb_2\ell+L\alpha^3T^2(b_2+b_p\rho).
 \label{eq:sharp-Phi-Xi}
\end{align}
If
\begin{equation}\label{eq:sharp-main-condition}
 0<\alpha\le L^{-1},\qquad L\alpha T\ge1,\qquad
 \Xi_T\le(4\Kprob)^{-1},
\end{equation}
then, with probability at least $1-\delta$,
\begin{align}
 \boxed{\ A_T\le3\Hprob\Phi_0\ },&\qquad
 \max_{t\le T+1}\{F(x_t)-F_*\}\le\tfrac32\Hprob\alpha T\Phi_0,
 \label{eq:sharp-main-bound}\\
 \max_{t\le T+1}\|x_t-x_1\|^2&\le R_\sharp^2:=2\Kprob\alpha^2T^2\Phi_0.
 \label{eq:sharp-radius}
\end{align}
All constants are given in~\eqref{eq:energy-constant} and~\eqref{eq:refined-constants} and depend only on $p$.
\end{theorem}
\begin{proof}
Suppose $\Phi_0>0$. Stop the original errors at the deterministic radius $R_\sharp$ exactly as in~\eqref{eq:distance-stop}, and use the auxiliary recursion~\eqref{eq:auxiliary}. Its conditional envelopes are
$s_R^2=\sigma^2+b_2R_\sharp^2$ and $n_R^2=\nu_p^2+b_pR_\sharp^2$.
Put
\[
 \Phi_1=L\alpha(b_2+b_p\rho)+b_2\ell/T.
\]
Then $\Phi(s_R,n_R)=\Phi_0+\Phi_1R_\sharp^2$ and
$\alpha^2T^2\Phi_1=\Xi_T$.
Proposition~\ref{prop:sharp-uniform} gives an event of probability at least $1-\delta$ on which the auxiliary maximal displacement is bounded by
\[
 \Kprob\alpha^2T^2(\Phi_0+\Phi_1R_\sharp^2)
 \le R_\sharp^2/2+R_\sharp^2/4.
\]
The first-exit induction from Theorem~\ref{thm:main} now proves equality with the original recursion throughout the run, including the update that would produce an exit. Also,
\[
 \Phi_1R_\sharp^2=2\Kprob\Xi_T\Phi_0\le\Phi_0/2.
\]
Insert this into the auxiliary stationarity and objective bounds in~\eqref{eq:sharp-envelope-conclusions} to obtain~\eqref{eq:sharp-main-bound}. If $\Phi_0=0$, the initial gap and both noise intercepts vanish; induction gives a constant trajectory.
\end{proof}
Expanding~\eqref{eq:sharp-radius}, its radius has order
$\alpha T\Delta+\alpha^2T\sigma^2\ell
 +L\alpha^3T^2(\sigma^2+\nu_p^2\rho)$.
In contrast with the baseline theorem, the cubic part of the feedback condition has no variance logarithm. The extra condition $L\alpha T\ge1$ absorbs the short-run transform term~\eqref{eq:U-absorption}; Proposition~\ref{prop:sharp-uniform} retains that term explicitly for every horizon.

\subsection{Trajectory control and admissibility}
Theorems~\ref{thm:expectation}, \ref{thm:main}, and~\ref{thm:sharp-main} control the maximum displacement over the entire run. Theorem~\ref{thm:expectation} bounds its second moment through Doob's inequality and requires no moment above order two. Equation~\eqref{eq:max-markov-localization} converts that expectation bound into a pathwise localization, with a $\delta^{-1}$ cost for the squared radius. Theorems~\ref{thm:main} and~\ref{thm:sharp-main} instead use the additional conditional $p$th moments to obtain a joint trajectory and stationarity event. In the refined theorem the leading variance baseline is $L\alpha\sigma^2$, while its logarithmic fluctuation term is $\sigma^2\ell/T$; the polynomial rare-shock term is retained. These statements have different moment assumptions and confidence costs. Neither imposes a fixed compact set uniformly over all horizons: both the expectation allowance and the radius~\eqref{eq:main-radius} may grow with $T$.

The condition~\eqref{eq:main-condition} is also part of the theorem. Reducing the failure probability increases both $\ell$ and $\rho$, and can force a smaller stepsize. Thus~\eqref{eq:main-bound} is not obtained from a deterministic-envelope theorem by weakening an assumption while keeping the entire admissible parameter region unchanged.

When $b_2=b_p=0$, the feedback condition vanishes. Proposition~\ref{prop:uniform} then gives the usual deterministic-envelope rate with a step of order $T^{-1/2}$. The more conservative distance-localized bound remains valid, but its radius is unnecessary in this special case.

\section{Stepsizes, confidence levels, and output criteria}\label{sec:rates}
We give quantitative consequences rather than suppressing the feedback restriction in asymptotic notation. Unless otherwise indicated, $p,L,\Delta,\sigma,\nu_p,b_2,b_p$ are fixed.

\subsection{A direct finite-horizon stepsize recipe}
Put $c_*=(96\Dprob)^{-1}$. If $W_1>0$, the three inequalities
\begin{equation}\label{eq:explicit-step}
 \alpha\le\min\left\{\frac1L,
       \left(\frac{c_*}{2TW_1}\right)^{1/2},
       \left(\frac{c_*}{2LT^2W_1}\right)^{1/3}\right\}
\end{equation}
imply~\eqref{eq:main-condition}, because
\[
 h_TW_1=\alpha^2TW_1+L\alpha^3T^2W_1\le c_*/2+c_*/2.
\]
For $W_1=0$, only the smoothness cap remains. The parameters in~\eqref{eq:explicit-step} are theoretical upper bounds on noise scales and growth. The constants are explicit but conservative. For example,
\begin{equation}\label{eq:numerical-constants}
 \mathfrak c_4=\frac{319}{120},\qquad
 \mathfrak D_4=2+\left(\frac{319}{120}\right)^2\sqrt6
             \simeq19.3099,\qquad
 c_*\simeq5.39447\times10^{-4}.
\end{equation}
For comparison, $\mathfrak c_3=41/30$ and $\mathfrak D_3\simeq8.16727$. Thus the displayed sufficient stepsizes can be small. They certify the worst-case guarantee, not the practical stability boundary. Raw-energy and objective localization, the square-of-a-sum estimate, and the subsequent feedback absorptions all contribute slack. The unabsorbed second-moment bound in Remark~\ref{rem:resolvent} can be used when a sharper numerical certificate is needed; adaptive estimation of growth parameters is outside the present analysis.

For the refined theorem, put $c_\sharp=(4\Kprob)^{-1}$. A sufficient upper stepsize recipe is
\begin{equation}\label{eq:sharp-step-recipe}
 \alpha\le\min\left\{L^{-1},
 \left(\frac{c_\sharp}{2Tb_2\ell}\right)^{1/2},
 \left(\frac{c_\sharp}{2LT^2(b_2+b_p\rho)}\right)^{1/3}\right\},
\end{equation}
with any zero-denominator restriction omitted, together with $L\alpha T\ge1$. The first growth cap controls $\alpha^2Tb_2\ell$ and the second controls the cubic term. This changes the placement of logarithms, not the need for confidence-sensitive stability. The constants $\Hprob,\Kprob$ are not asserted to improve every numerical prefactor of the baseline bound.

\begin{corollary}[An interpolating finite-horizon bound]\label{cor:interpolation}
Suppose $\Delta>0$ and $W_0>0$. Choose $\alpha$ equal to the minimum in~\eqref{eq:explicit-step} and the additional value
\[
 \left(\frac{\Delta}{LTW_0}\right)^{1/2},
\]
omitting growth-dependent restrictions when $W_1=0$. Then, with probability at least $1-\delta$,
\begin{equation}\label{eq:interpolation}
 A_T\le C_p\left\{
     \frac{L\Delta}{T}
     +\sqrt{\frac{L\Delta W_0}{T}}
     +\Delta\sqrt{\frac{W_1}{T}}
     +\Delta\left(\frac{LW_1}{T}\right)^{1/3}
              \right\},
\end{equation}
where $C_p$ depends only on $p$.
\end{corollary}
\begin{proof}
The choice satisfies the main theorem and $L\alpha W_0\le\sqrt{L\Delta W_0/T}$. Since the reciprocal of a minimum of positive numbers is the maximum of their reciprocals,
\begin{align*}
 \frac{\Delta}{\alpha T}
 =\max\left\{
    \frac{L\Delta}{T},
    \Delta\sqrt{\frac{2W_1}{c_*T}},
    \Delta\left(\frac{2LW_1}{c_*T}\right)^{1/3},
    \sqrt{\frac{L\Delta W_0}{T}}
          \right\}.
\end{align*}
Bound the maximum by the sum and apply~\eqref{eq:main-bound}. Since $c_*$ and $\Dprob$ depend only on $p$, their numerical factors can be included in $C_p$.
\end{proof}
When $\Delta=0$ or $W_0=0$, use Theorem~\ref{thm:main} and~\eqref{eq:explicit-step} directly, rather than dividing by a vanishing quantity. In particular, $\Delta=0$ does not mean that the trajectory is constant when the noise intercept is nonzero.

\subsection{Fixed confidence}
\begin{corollary}[Finite-$p$ fixed-confidence rate]\label{cor:fixed-confidence}
Fix $\delta\in(0,1)$ and let $\ell=\log(6/\delta)$. Define
\[
 \overline W_1=b_2\ell+b_p\delta^{-2/p}.
\]
If ${a_0}>0$ satisfies ${a_0}\le L^{-1}$ and
\begin{equation}\label{eq:fixed-a}
 ({a_0}^2+L{a_0}^3)\overline W_1\le\frac1{96\Dprob},
\end{equation}
then $\alpha={a_0}T^{-2/3}$ is admissible for every positive integer $T$, and with probability at least $1-\delta$,
\begin{equation}\label{eq:fixed-p-rate}
 A_T\le\frac{6\Delta}{{a_0}}T^{-1/3}
       +18\Dprob L{a_0}\sigma^2\ell\,T^{-2/3}
       +18\Dprob L{a_0}\nu_p^2\delta^{-2/p}T^{2/p-5/3}.
\end{equation}
Consequently $A_T=O(T^{-1/3})$ at the stated fixed confidence.
\end{corollary}
\begin{proof}
Since $p>2$ and $T\ge1$, $T^{2/p-1}\le1$, so $W_1\le\overline W_1$. Also $h_T={a_0}^2T^{-1/3}+L{a_0}^3$. Thus~\eqref{eq:fixed-a} implies~\eqref{eq:main-condition}. Substitute the stepsize into~\eqref{eq:main-bound}. Finally $2/p-5/3<-2/3$, so the last two powers decay faster than $T^{-1/3}$.
\end{proof}
The finite-$p$ assumption is not needed for the exponent $1/3$ at fixed confidence: Corollary~\ref{cor:expect-rate} already supplies it under (M2). Its role is the confidence decomposition, particularly the stronger increasing-confidence guarantees derived next from Theorem~\ref{thm:sharp-main}.

\subsection{Polynomially increasing confidence without a logarithmic rate loss}
\begin{corollary}[Increasing confidence]\label{cor:increasing}
Fix $r>0$, set $\delta_T=T^{-r}$ for $T\ge2$, and define
\begin{equation}\label{eq:confidence-w}
 z=\frac{2(1+r)}p-1,\qquad w_T=1+T^z.
\end{equation}
If $0<r<p-1$, there is a sufficiently small fixed $a_0>0$ such that
\begin{equation}\label{eq:increasing-step}
 \alpha_T=a_0T^{-2/3}w_T^{-1/3}
\end{equation}
satisfies~\eqref{eq:sharp-main-condition} for all sufficiently large $T$. For these fixed-horizon runs,
\begin{equation}\label{eq:increasing-rate}
 \Prob\!\left\{A_T>C(w_T/T)^{1/3}\right\}\le T^{-r},
\end{equation}
where $C$ depends only on fixed problem parameters, $p,r,a_0$. In particular,
\begin{equation}\label{eq:log-confidence}
 0<r\le p/2-1\quad\Longrightarrow\quad
 A_T=O(T^{-1/3})\quad\text{with failure probability }T^{-r}.
\end{equation}
In this range $\alpha_T=a_0T^{-2/3}$ itself is admissible for a sufficiently small fixed $a_0$; no logarithmic adjustment is necessary. If $p/2-1<r<p-1$, a sufficient rate is
\begin{equation}\label{eq:high-confidence-power}
 A_T=O\!\left(T^{-[2-2(1+r)/p]/3}\right)
 \quad\text{with failure probability }T^{-r}.
\end{equation}
\end{corollary}
\begin{proof}
Here $\rho=T^z$, $\ell=\log(6T^r)$, and $b_2+b_p\rho\le B w_T$ with $B=b_2+b_p$. Thus
\begin{equation}\label{eq:increasing-admissibility}
 \Xi_T\le a_0^2 b_2T^{-1/3}w_T^{-2/3}\log(6T^r)+La_0^3B.
\end{equation}
The first term tends to zero because $w_T\ge1$. Choose $a_0$ so that the second is less than $(8\Kprob)^{-1}$ when $B>0$; if $B=0$ the condition is empty. Since $r<p-1$ is equivalent to $z<1$, we have $w_T/T\to0$ and
$L\alpha_TT=La_0(T/w_T)^{1/3}\to\infty$. Also $\alpha_T\to0$, so the two remaining conditions of Theorem~\ref{thm:sharp-main} hold eventually.

With $S=\sigma^2+\nu_p^2$, the terms in $\Phi_0$ satisfy
\[
 \frac{\Delta}{\alpha_TT}=\frac\Delta{a_0}(w_T/T)^{1/3},\qquad
 L\alpha_T(\sigma^2+\nu_p^2\rho)\le La_0S(w_T/T)^{2/3}.
\]
The remaining term $\sigma^2\log(6T^r)/T=o(T^{-1/3})$.
Theorem~\ref{thm:sharp-main} proves~\eqref{eq:increasing-rate}. If $z\le0$, then $w_T$ is bounded above and below by positive constants; the same calculation works with $w_T=1$ in the stepsize. If $0<z<1$, then $w_T\asymp T^z$, proving~\eqref{eq:high-confidence-power}.
\end{proof}
The exponent range $r<p-1$ is sufficient for this construction. A fixed-horizon probability statement at each $T$ is not a time-uniform event for one run; the decreasing-stepsize theorem in Section~\ref{sec:single-run} provides a different type of guarantee.

\begin{example}[A fourth-moment certificate]\label{ex:p4}
For $p=4$ and $\delta_T=T^{-1}$, $\rho=1$. A sufficiently small step $\alpha=a_0T^{-2/3}$ gives
\[
 \Prob\{A_T>CT^{-1/3}\}\le T^{-1}
\]
for all sufficiently large $T$. The logarithmic contribution in~\eqref{eq:sharp-main-bound} is only $O(\log T/T)$, smaller than $T^{-1/3}$. By contrast, converting the $O(T^{-1/3})$ expectation bound by first-moment Markov at failure probability $T^{-1}$ gives an $O(T^{2/3})$ certificate. The finite-$p$ confidence refinement removes that loss without changing the update.
\end{example}

\subsection{Stationarity criteria and oracle complexity}\label{sec:output}
Let $\tau$ be uniform on $\{1,\ldots,T\}$ and independent of the complete run. Then
\begin{equation}\label{eq:random-output}
 \E\bigl(\norm{\nabla F(x_\tau)}^2\mid x_1,\ldots,x_T\bigr)=A_T,
 \qquad \E\norm{\nabla F(x_\tau)}^2=\E A_T.
\end{equation}
Corollary~\ref{cor:complexity} gives a parameter-explicit sufficient budget for this squared-gradient output criterion. Each update uses exactly one oracle call; the expectation guarantee already implies the first-moment stationarity criterion used in the external BG-0 lower bound. Full-gradient evaluations used to report $A_T$ are not part of the optimizer and are unnecessary for selecting $x_\tau$.

For a probability bound, if $\Prob(A_T>B)\le\delta$ and $u\in(0,1)$, conditional Markov and~\eqref{eq:random-output} give
\begin{equation}\label{eq:output-tail}
 \Prob\{\norm{\nabla F(x_\tau)}^2>B/u\}\le\delta+u.
\end{equation}
Indeed, on $\{A_T\le B\}$ the conditional probability is at most $u$, and the complementary event has probability at most $\delta$. Also $\min_{t\le T}\norm{\nabla F(x_t)}^2\le A_T$ pathwise. Neither observation is a last-iterate theorem or an assertion that the iterate minimizing the full gradient norm is available without additional computation.

\section{Objective-gap moment growth and a \texorpdfstring{root-$T$}{root-T} refinement}\label{sec:gap}
Distance growth alone allows noise to increase along flat objective directions. When the noise scale is controlled directly by the objective gap, a sharper localization condition is available. This is an additional assumption, not a consequence of the two-point Lipschitz condition by itself.

\begin{assumption}[Objective-gap moment growth]\label{ass:gap}
For $p>2$ there are deterministic $\omega_2,\omega_p,\gamma_2,\gamma_p\ge0$ such that, almost surely,
\begin{align}
 \E_{t-1}\norm{\xi_t}^2
 &\le \omega_2+\gamma_2\{F(x_t)-F_*\},\notag\\
 \bigl(\E_{t-1}\norm{\xi_t}^p\bigr)^{2/p}
 &\le \omega_p+\gamma_p\{F(x_t)-F_*\}.
 \label{eq:gap-growth}
\end{align}
\end{assumption}
The intercepts $\omega_2,\omega_p$ are squared moment scales; $a_0$ denotes a horizon-dependent stepsize coefficient. For this assumption the finite-horizon integrability argument remains valid: the smooth upper estimate
\[
 F(x)-F_*\le\Delta+\norm{\nabla F(x_1)}\norm{x-x_1}+\frac L2\norm{x-x_1}^2
\]
converts the squared root-moment bound to a constant times $1+\norm{x-x_1}^2$, so Lemma~\ref{lem:integrability}'s induction applies.

\Needspace{21\baselineskip}
\begin{theorem}[Nagaev bound under objective-gap growth]\label{thm:gap}
Suppose Assumptions~\ref{ass:objective} and~\ref{ass:gap} hold. Define
\begin{equation}\label{eq:Q}
 Q_0=\omega_2\ell+\omega_p\rho,\qquad Q_1=\gamma_2\ell+\gamma_p\rho,
\end{equation}
with $\ell,\rho$ as in~\eqref{eq:ell-rho}. If
\begin{equation}\label{eq:gap-condition}
 0<\alpha\le L^{-1},\qquad L\alpha^2TQ_1\le\frac1{24\Dprob},
\end{equation}
then, with probability at least $1-\delta$,
\begin{align}
 A_T&\le\frac{8\Delta}{\alpha T}+18\Dprob L\alpha Q_0,
 \label{eq:gap-bound}\\
 \max_{1\le t\le T+1}\{F(x_t)-F_*\}
 &\le H:=8\Delta+12\Dprob L\alpha^2TQ_0.
 \label{eq:gap-H}
\end{align}
\end{theorem}
\begin{proof}
Assume $H>0$ and define predictable indicators on the original trajectory by
\[
 I_t=\1_{\{\max_{1\le s\le t}(F(x_s)-F_*)\le H\}},\qquad
 \widetilde\xi_t=I_t\xi_t.
\]
Their deterministic envelopes are $s_H^2=\omega_2+\gamma_2H$ and $n_H^2=\omega_p+\gamma_pH$. Drive the auxiliary recursion~\eqref{eq:auxiliary} with these errors. Proposition~\ref{prop:uniform} provides a good event of probability at least $1-\delta$ and a deterministic $J_H$ obeying
\[
 J_H^2\le\Dprob T(Q_0+Q_1H).
\]
On this event,
\begin{align*}
 \max_t\{F(\widetilde x_t)-F_*\}
 &\le4\Delta+6\Dprob L\alpha^2T(Q_0+Q_1H)\\
 &=\frac H2+6\Dprob L\alpha^2TQ_1H\le\frac{3H}{4}.
\end{align*}
If the original trajectory first exited the objective level $H$ at time $\tau$, the errors and recursions would agree through the update producing $x_\tau$. Then $F(x_\tau)=F(\widetilde x_\tau)\le F_*+3H/4$, a contradiction. The trajectories therefore agree on the good event and~\eqref{eq:gap-H} follows.

Their stationarity bound satisfies
\begin{align*}
 A_T
 &\le\frac{4\Delta}{\alpha T}+12\Dprob L\alpha(Q_0+Q_1H)\\
 &=\frac{4\Delta}{\alpha T}+12\Dprob L\alpha Q_0
   +96\Dprob L\alpha Q_1\Delta
   +144\Dprob^2 L^2\alpha^3TQ_1Q_0.
\end{align*}
By~\eqref{eq:gap-condition}, the last two terms are at most $4\Delta/(\alpha T)$ and $6\Dprob L\alpha Q_0$, respectively. This proves~\eqref{eq:gap-bound}. If $H=0$, then $\Delta=Q_0=0$, so the gradient and the noise at initialization vanish. Induction gives a constant trajectory, proving the degenerate case.
\end{proof}

\begin{corollary}[Fixed-confidence root-$T$ rate]\label{cor:rootT}
Fix $\delta\in(0,1)$ and all problem parameters in Assumption~\ref{ass:gap}. A sufficiently small fixed ${a_0}>0$ and $\alpha={a_0}T^{-1/2}$ satisfy~\eqref{eq:gap-condition} for all sufficiently large $T$. For these runs,
\begin{equation}\label{eq:rootT}
 A_T=O(T^{-1/2})\quad\hbox{with probability at least }1-\delta.
\end{equation}
\end{corollary}
\begin{proof}
For $T\ge1$, $Q_1\le\gamma_2\ell+\gamma_p\delta^{-2/p}$. Hence
\[
 L\alpha^2TQ_1\le L{a_0}^2(\gamma_2\ell+\gamma_p\delta^{-2/p}),
\]
which can be made at most $(24\Dprob)^{-1}$. The smoothness cap eventually holds. Substituting in~\eqref{eq:gap-bound}, the transient and variance terms have order $T^{-1/2}$ and the polynomial term has order $T^{2/p-3/2}=o(T^{-1/2})$.
\end{proof}

\begin{proposition}[A geometric sufficient condition]\label{prop:gap-geometry}
Suppose the pointwise iid-oracle version of Assumption~\ref{ass:p} holds. If deterministic $r_0,\kappa\ge0$ satisfy
\begin{equation}\label{eq:gap-geometry}
 \norm{x-x_1}^2\le r_0^2+\kappa\{F(x)-F_*\},\qquad x\in\R^d,
\end{equation}
then Assumption~\ref{ass:gap} holds with
\[
 \omega_2=\sigma^2+b_2r_0^2,\quad \omega_p=\nu_p^2+b_pr_0^2,
 \qquad \gamma_2=b_2\kappa,\quad\gamma_p=b_p\kappa.
\]
Alternatively, if a minimizer $x_*$ satisfies
$F(x)-F_*\ge(\mu/2)\norm{x-x_*}^2$ for $\mu>0$, and
$\norm{e(x,Z)}_q\le s_{q,*}+K_q\norm{x-x_*}$ for $q=2,p$, then one may take
\begin{equation}\label{eq:quadratic-growth-moments}
 \omega_q=2s_{q,*}^2,\qquad\gamma_q=4K_q^2/\mu,\qquad q=2,p.
\end{equation}
The same constants are valid for a nonempty minimizer set $\mathcal X_*=\operatorname{argmin}F$, possibly unbounded, if instead
\begin{equation}\label{eq:minimizer-set-growth}
 F(x)-F_*\ge\frac\mu2\dist(x,\mathcal X_*)^2,\qquad
 \sup_{z\in\mathcal X_*}\|e(z,Z)\|_q\le s_{q,*},
\end{equation}
and $\|e(x,Z)-e(y,Z)\|_q\le K_q\|x-y\|$ for $q=2,p$.
\end{proposition}
\begin{proof}
Substitute~\eqref{eq:gap-geometry} into the two pointwise squared root-moment bounds, and use the fresh-sample identity~\eqref{eq:conditional-substitution}. For the second assertion,
\[
 \norm{e(x,Z)}_q^2
 \le2s_{q,*}^2+2K_q^2\norm{x-x_*}^2
 \le2s_{q,*}^2+\frac{4K_q^2}{\mu}\{F(x)-F_*\}.
\]
Conditioning gives~\eqref{eq:quadratic-growth-moments}. For the set-valued assertion, $\mathcal X_*$ is closed by continuity of $F$, so each deterministic $x$ has a nearest point $z\in\mathcal X_*$. The triangle inequality gives
\[
 \|e(x,Z)\|_q\le s_{q,*}+K_q\dist(x,\mathcal X_*).
\]
Square and apply~\eqref{eq:minimizer-set-growth}. This is a pointwise moment argument, so no measurable nearest-point selection along the random trajectory is needed before using~\eqref{eq:conditional-substitution}.
\end{proof}
Strong convexity is sufficient for the second geometric inequality when $F_*$ is the minimum, but is not necessary. Neither ordinary nonconvex smoothness nor a bound on the objective alone implies~\eqref{eq:gap-geometry}. Quadratic growth relative to a minimizer set does not bound distance to a particular point of that set. The set-valued verification avoids that step by requiring a uniform noise-moment bound only on the minimizer set.

\paragraph{Relation to expected smoothness.}
In the iid setting, Assumption~\ref{ass:gap} at $q=2$ gives
\[
 \E_Z\norm{G(x,Z)}^2
 \le\norm{\nabla F(x)}^2+\omega_2+\gamma_2\{F(x)-F_*\}.
\]
This is a special case of the ABC condition of \citet{khaled2023}. Theorem~\ref{thm:gap} adds a specified $p$th-moment growth assumption and an explicit variance/rare-shock confidence bound. Its contribution is the finite-moment confidence decomposition under this established second-moment framework.

\subsection{Refined gap-growth confidence bounds}
\begin{theorem}[Energy-refined objective-gap bound]\label{thm:sharp-gap}
Suppose Assumptions~\ref{ass:objective} and~\ref{ass:gap} hold, and put
\begin{align}
 \Psi_0&=\frac{\Delta}{\alpha T}
 +L\alpha(\omega_2+\omega_p\rho)+\frac{\omega_2\ell}{T},\notag\\
 \Gamma_T&=L\alpha^2T(\gamma_2+\gamma_p\rho)+\alpha\gamma_2\ell.
 \label{eq:sharp-gap-parameters}
\end{align}
If
\begin{equation}\label{eq:sharp-gap-condition}
 0<\alpha\le L^{-1},\qquad L\alpha T\ge1,\qquad
 \Gamma_T\le(4\Hprob)^{-1},
\end{equation}
then, with probability at least $1-\delta$,
\begin{equation}\label{eq:sharp-gap-bound}
 \boxed{\ A_T\le3\Hprob\Psi_0\ },\qquad
 \max_{t\le T+1}\{F(x_t)-F_*\}\le2\Hprob\alpha T\Psi_0.
\end{equation}
\end{theorem}
\begin{proof}
For $\Psi_0>0$, stop the original errors at the deterministic objective level
$H_\sharp=2\Hprob\alpha T\Psi_0$. Their envelopes are
$s_H^2=\omega_2+\gamma_2H_\sharp$ and
$n_H^2=\omega_p+\gamma_pH_\sharp$.
Apply Proposition~\ref{prop:sharp-uniform} to the corresponding auxiliary recursion. It bounds the auxiliary maximum objective gap by
\[
 \Hprob\alpha T\Psi_0+\Hprob\Gamma_T H_\sharp
 \le H_\sharp/2+H_\sharp/4.
\]
The same first-exit induction as in Theorem~\ref{thm:gap} proves agreement with the original process. The auxiliary stationarity bound is
\[
 2\Hprob\left\{\Psi_0+\frac{\Gamma_T}{\alpha T}H_\sharp\right\}
 =2\Hprob(1+2\Hprob\Gamma_T)\Psi_0\le3\Hprob\Psi_0.
\]
If $\Psi_0=0$, the initial gap and both intercepts vanish, so the trajectory is constant.
\end{proof}

\begin{corollary}[Increasing confidence under objective-gap growth]\label{cor:sharp-gap-confidence}
Fix $0<r<p-1$, put $\delta_T=T^{-r}$ and $w_T=1+T^{2(1+r)/p-1}$. A sufficiently small fixed $a_0>0$ and
\begin{equation}\label{eq:sharp-gap-step}
 \alpha_T=a_0T^{-1/2}w_T^{-1/2}
\end{equation}
give, for all sufficiently large $T$,
\begin{equation}\label{eq:sharp-gap-confidence}
 \Prob\{A_T>C(w_T/T)^{1/2}\}\le T^{-r}.
\end{equation}
In particular, if $0<r\le p/2-1$, the unadjusted step $\alpha_T=a_0T^{-1/2}$ yields
\begin{equation}\label{eq:rootT-increasing}
 A_T=O(T^{-1/2})\quad\text{with failure probability }T^{-r}.
\end{equation}
\end{corollary}
\begin{proof}
Since $\rho=T^{2(1+r)/p-1}$, the first term in $\Gamma_T$ is at most
$La_0^2(\gamma_2+\gamma_p)$, and its second term tends to zero. A sufficiently small $a_0$ gives the feedback cap. Also $w_T/T\to0$, so $L\alpha_TT\to\infty$ and $\alpha_T\to0$. In $\Psi_0$, both the transient and $L\alpha_T(\omega_2+\omega_p\rho)$ are $O((w_T/T)^{1/2})$, whereas $\omega_2\log(6T^r)/T=o(T^{-1/2})$. Apply Theorem~\ref{thm:sharp-gap}. In the stated root-$T$ range, $w_T$ is bounded and the proof also applies to $\alpha_T=a_0T^{-1/2}$.
\end{proof}

\section{Examples and the necessity of noise-sensitive stability}\label{sec:examples}
All random observations in this section are independent across iterations. The examples verify the assumptions pointwise in the parameter before conditional substitution along the trajectory.

\subsection{Random affine gradient errors}
\begin{example}[Matrix-valued multiplicative noise]\label{ex:affine}
Let $B=B(Z)$ be a random $d\times d$ matrix and $\eta=\eta(Z)$ a random vector, with $\E B=0$, $\E\eta=0$, $\norm{\norm B_{\op}}_p<\infty$, and $\norm\eta_p<\infty$. No independence between $B$ and $\eta$ is required. Define
\begin{equation}\label{eq:affine-oracle}
 G(x,Z)=\nabla F(x)+B(Z)(x-x_1)+\eta(Z).
\end{equation}
For $q=2,p$, Minkowski gives
\[
 \norm{e(x,Z)}_q
 \le\norm\eta_q+\norm{\norm B_{\op}}_q\norm{x-x_1}.
\]
Thus the squared root-moment growth constants may be chosen as
\[
 \sigma^2=2\norm\eta_2^2,\quad b_2=2\norm{\norm B_{\op}}_2^2,
 \qquad
 \nu_p^2=2\norm\eta_p^2,\quad b_p=2\norm{\norm B_{\op}}_p^2.
\]
A sharper second-moment choice is available when $\E[B^\top\eta]=0$. In that case, for $d=x-x_1$,
\[
 \E\|Bd+\eta\|^2=d^\top\E[B^\top B]d+\E\|\eta\|^2,
\]
so one can take
\begin{equation}\label{eq:affine-spectral-constants}
 \sigma^2=\E\|\eta\|^2,\qquad
 b_2=\lambda_{\max}(\E[B^\top B]).
\end{equation}
This is an optional orthogonality condition; the preceding moment bounds remain valid without it. The spectral coefficient can be smaller than $\E\|B\|_{\op}^2$ and avoids the factor two.

The full stochastic gradient is $L^p$-Lipschitz with constant at most $L+\norm{\norm B_{\op}}_p$. If $B$ is symmetric, the oracle is the gradient of
\[
 f(x,Z)=F(x)+\frac12(x-x_1)^\top B(Z)(x-x_1)+\eta(Z)^\top(x-x_1).
\]
The objective is $\E f(x,Z)=F(x)$. Individual random losses are not required to be bounded below. If there is a direction $v$ with $\norm{Bv}_q>0$, the reverse triangle inequality along $x=x_1+rv$ gives
$\norm{e(x,Z)}_q\ge |r|\norm{Bv}_q-\norm\eta_q$, so a global moment envelope is impossible.
\end{example}

\subsection{Nonconvex flat directions: distance growth without gap growth}
\begin{example}[An objective level cannot control the noise]\label{ex:flat}
Let $Z$ be centered, nondegenerate, and in $L^p$. On $\R^2$, take
\begin{equation}\label{eq:flat-example}
 F(u,v)=1-\cos u,\qquad
 f((u,v),Z)=1-\cos u+Zuv,
 \qquad x_1=(\pi/2,0).
\end{equation}
Then $F_*=0$, $L=1$, and $F$ is nonconvex because its Hessian is $\operatorname{diag}(\cos u,0)$. The oracle and its error are
\[
 G((u,v),Z)=(\sin u+Zv,\,Zu),\qquad e((u,v),Z)=Z(v,u).
\]
Consequently
\begin{equation}\label{eq:flat-moments}
 \norm{e((u,v),Z)}_q=\norm Z_q\sqrt{u^2+v^2},\qquad q=2,p,
\end{equation}
and $G$ is $L^p$-Lipschitz with constant at most $1+\norm Z_p$. Relative to $x_1$, squared root-moment growth holds with intercept $2\norm Z_q^2(\pi/2)^2$ and coefficient $2\norm Z_q^2$. Theorem~\ref{thm:main} therefore applies for its stated stepsizes.

At $(u,v)=(0,v)$, however, $F(0,v)-F_*=0$ while $\norm{e((0,v),Z)}_q=|v|\norm Z_q$. No finite constants $\omega_q,\gamma_q$ can satisfy the global pointwise gap-growth inequality. The example has a bounded objective, no coercivity, and unbounded noise moments along a minimum-level set. It separates the two localization theorems without relying on an abstract counterexample to compactness.
\end{example}

\subsection{A nonconvex coercive example with gap growth}
\begin{example}[Quadratic growth without convexity]\label{ex:coercive}
Let
\begin{equation}\label{eq:coercive}
 F(x)=\frac{x^2}{2}+2(1-\cos x),\qquad x\in\R.
\end{equation}
Here $F_*=0$, $F''(x)=1+2\cos x$, and $L=3$ is a valid smoothness constant. Since $F''(\pi)=-1$, the objective is nonconvex. Nevertheless, $F(x)\ge x^2/2$ and
\[
 |x-x_1|^2\le2x^2+2x_1^2\le4F(x)+2x_1^2.
\]
Thus~\eqref{eq:gap-geometry} holds with $r_0^2=2x_1^2$ and $\kappa=4$. Any affine-noise oracle in Example~\ref{ex:affine} satisfies the gap-growth theorem, and sufficiently small steps of order $T^{-1/2}$ yield the fixed-confidence root-$T$ bound. Convexity is not used in this verification.
\end{example}

\subsection{Random-design regression with a nonconvex penalty}
\begin{example}[Unbounded random design]\label{ex:regression}
Let $(a,y)\in\R^d\times\R$ satisfy
\begin{equation}\label{eq:regression-moments}
 \E\norm a^{2p}<\infty,\qquad
 \E(|y|\norm a)^p<\infty,\qquad \E y^2<\infty,
 \qquad p>2.
\end{equation}
For $\lambda\ge0$, consider
\begin{align}
 F(x)&=\frac12\E(a^\top x-y)^2
             +\lambda\sum_{j=1}^d(1-\cos x_j),\notag\\
 G(x,(a,y))&=a(a^\top x-y)+\lambda\sin x,
 \label{eq:regression-oracle}
\end{align}
where the sine is coordinatewise. Put $Q=\E aa^\top$ and $c=\E ay$. The moment assumptions imply these quantities are finite and justify differentiating the quadratic expectation directly by expansion. Thus $\nabla F(x)=Qx-c+\lambda\sin x$, $F\ge0$, and $L=\norm Q_{\op}+\lambda$ is a valid smoothness bound whenever positive.

Write
\[
 B=aa^\top-Q,\qquad
 \eta=Bx_1-(ay-c).
\]
Then $\E B=0$, $\E\eta=0$, and the error equals $B(x-x_1)+\eta$. The moment conditions~\eqref{eq:regression-moments} imply the $L^p$ operator-norm and vector-moment conditions in Example~\ref{ex:affine}. Indeed, $\norm{aa^\top}_{\op}=\norm a^2$, and Minkowski bounds both $B$ and $\eta$. No almost-sure boundedness of the design or response is required.

This model can be genuinely nonconvex: in dimension one with $Q=1$ and $\lambda=2$, the Hessian is $1+2\cos x$, which is negative at $x=\pi$. The error is nevertheless unbiased and stochastic Lipschitz continuity holds. Applying the distance-growth theorem does not require a smallest eigenvalue bound on $Q$ or coercivity of this objective.
\end{example}

\subsection{A smoothness-only stepsize cap is insufficient}
\begin{proposition}[Multiplicative instability]\label{prop:instability}
There is a one-dimensional strongly convex objective and an unbiased iid stochastic-gradient oracle that is $L^p$-Lipschitz for every finite $p$, but whose SGD iterates diverge on every sample path at a stepsize satisfying $\alpha\le(8L)^{-1}$.
\end{proposition}
\begin{proof}
Let $F(x)=x^2/2$, $x_1=1$, and let $Z_t$ be iid symmetric signs. Define
\[
 G(x,Z)=(1+24Z)x,\qquad \alpha=1/8.
\]
Then $L=1$, $\E G(x,Z)=x$, $G(0,Z)=0$, and
\[
 \norm{G(x,Z)-G(y,Z)}_p=\norm{1+24Z}_p|x-y|<\infty
\]
for every finite $p$. The recursion is
\[
 x_{t+1}=\left(\frac78-3Z_t\right)x_t.
\]
The two multipliers are $-17/8$ and $31/8$. Therefore
\[
 |x_t|\ge(17/8)^{t-1},\qquad
 A_T=\frac1T\sum_{t=1}^Tx_t^2\ge\frac1T(17/8)^{2(T-1)}
\]
on every sample path. This establishes the assertion.
\end{proof}
More generally, with $G(x,Z)=(1+cZ)x$ and symmetric signs,
\begin{equation}\label{eq:quadratic-stability}
 \E x_{t+1}^2=\{(1-\alpha)^2+c^2\alpha^2\}\E x_t^2.
\end{equation}
Mean-square contraction holds exactly when
$0<\alpha<2/(1+c^2)$, by expanding the factor and requiring it to be less than one. This calculation establishes the need for noise-sensitive stability, rather than an asymptotic rate lower bound for this example. The broad BG-0 minimax conclusion instead follows from the oracle-class comparison in Section~\ref{sec:minimax}.

\subsection{An unbounded minimizer set with controlled noise}
\begin{example}[Nonunique minimizers and a nonconvex objective]\label{ex:minimizer-set}
On $\R^2$, let
\[
 F(u,v)=\tfrac12u^2+2(1-\cos u),\qquad
 e((u,v),Z,\varepsilon)=(Zu+\varepsilon,0),
\]
where $Z,\varepsilon$ are centered with finite $p$th moments. They may be dependent within each iid observation. The oracle is the gradient of
$f((u,v),(Z,\varepsilon))=F(u,v)+Zu^2/2+\varepsilon u$.
The objective is $3$-smooth, nonconvex, and lower bounded by $u^2/2$; its minimizer set is $\mathcal X_*=\{0\}\times\R$. Thus
\[
 F(u,v)\ge\tfrac12\dist((u,v),\mathcal X_*)^2,
 \qquad
 \sup_{z\in\mathcal X_*}\|e(z,Z,\varepsilon)\|_q=\|\varepsilon\|_q.
\]
The centered field is $L^q$-Lipschitz with constant $\|Z\|_q$ for $q=2,p$. Proposition~\ref{prop:gap-geometry} therefore gives
$\omega_q=2\|\varepsilon\|_q^2$ and $\gamma_q=4\|Z\|_q^2$.
Both gap-growth theorems apply, including the increasing-confidence root-$T$ result, without coercivity on $\R^2$ or uniqueness of the minimizer.
\end{example}

\section{An additive rare-shock lower bound}\label{sec:lower}
The BG-0 oracle-complexity lower bound in Section~\ref{sec:minimax} and the confidence lower bound below concern different questions. The next elementary construction shows that unchanged SGD cannot have a variance-only logarithmic confidence certificate over a finite-$p$ class. It specializes the sparse-shock argument in the supplied manuscript identified in Section~\ref{sec:provenance}; the full calculation is included, and no separate originality claim is made for it. The same broad rare-event mechanism underlies the Nagaev sharpness phenomenon for linear averaged SGD in \citet{zhu2022}.

\begin{proposition}[Polynomial confidence cost for unchanged SGD]\label{prop:lower}
Let $F(x)=\lambda x^2/2$ with $\lambda>0$, $x_1=0$, $p>2$, and $\nu>0$. Suppose
\[
 {\vartheta}=\alpha\lambda\in(0,1/4],\qquad
 N=\left\lfloor\frac1{2{\vartheta}}\right\rfloor,
\]
and let $T$ be even with $T\ge2N$. For every $\delta\in(0,1/8]$, there is a centered iid additive-noise law with $\E|\xi_t|^p=\nu^p$ such that
\begin{equation}\label{eq:lower}
 \Prob\left\{
 A_T\ge\frac{\lambda\alpha\nu^2}{16\,4^{2/p}}
               T^{2/p-1}\delta^{-2/p}\right\}\ge\delta.
\end{equation}
The law may depend on $T$ and $\delta$. Its variance is
\begin{equation}\label{eq:lower-variance}
 \E\xi_t^2=\nu^2(4\delta/T)^{1-2/p}.
\end{equation}
\end{proposition}
\begin{proof}
Set $r_0=4\delta/T$ and $v=\nu r_0^{-1/p}$. Independently at each time, let
\[
 \Prob(\xi_t=v)=\Prob(\xi_t=-v)=r_0/2,
 \qquad\Prob(\xi_t=0)=1-r_0.
\]
The law is centered, $\E|\xi_t|^p=r_0v^p=\nu^p$, and~\eqref{eq:lower-variance} follows from $r_0v^2$.

Let $\mathcal E$ be the event that there is exactly one nonzero shock among times $1,\ldots,T$, located in $1,\ldots,T/2$. Then
\begin{align}
 \Prob(\mathcal E)
 &=\frac T2 r_0(1-r_0)^{T-1}
   =2\delta(1-4\delta/T)^{T-1}\notag\\
 &\ge2\delta\{1-(T-1)4\delta/T\}
   \ge2\delta(1-4\delta)\ge\delta.
 \label{eq:one-shock-probability}
\end{align}
We used Bernoulli's inequality and $\delta\le1/8$.

If the unique shock occurs at time $s\le T/2$, the quadratic recursion gives
\begin{equation}\label{eq:quadratic-impulse}
 x_{s+1+j}=-\alpha\xi_s(1-{\vartheta})^j,\qquad j\ge0
\end{equation}
as long as there are no later shocks, while $x_1=\cdots=x_s=0$. Since ${\vartheta}\le1/4$, $1/(2{\vartheta})\ge2$ and
\[
 N=\lfloor1/(2{\vartheta})\rfloor\ge1/(4{\vartheta}).
\]
For $0\le j<N$, Bernoulli's inequality gives $(1-{\vartheta})^j\ge1-{\vartheta}j\ge1/2$. All $N$ iterates in~\eqref{eq:quadratic-impulse} with these indices enter $A_T$, because $s+N\le T$. Therefore, on $\mathcal E$,
\[
 A_T\ge\frac{N\lambda^2\alpha^2v^2}{4T}
       \ge\frac{\lambda\alpha v^2}{16T}
       =\frac{\lambda\alpha\nu^2}{16\,4^{2/p}}
                      T^{2/p-1}\delta^{-2/p}.
\]
Combine this inequality with~\eqref{eq:one-shock-probability}.
\end{proof}

\begin{remark}[Exact finite-horizon response]\label{rem:exact-impulse}
On a path with its only shock of magnitude $v$ at time $s$, summing the geometric series in~\eqref{eq:quadratic-impulse} gives
\begin{equation}\label{eq:exact-impulse}
 A_T=\frac{\lambda\alpha v^2}{T(2-{\vartheta})}
       \left\{1-(1-{\vartheta})^{2(T-s)}\right\}.
\end{equation}
Indeed, before simplification the expression is
$\lambda^2\alpha^2v^2T^{-1}\sum_{j=0}^{T-s-1}(1-{\vartheta})^{2j}$,
whose denominator is $1-(1-{\vartheta})^2={\vartheta}(2-{\vartheta})$. When $s=T$ the sum is empty and both sides vanish. This indexing matters because $A_T$ uses the gradients before the updates.
\end{remark}

The additive oracle $G(x,Z)=\lambda x+Z$ has stochastic Lipschitz constant $\lambda$ independent of the noise law, and the actual centered error has $b_2=b_p=0$. Thus Proposition~\ref{prop:lower} applies to a subclass already covered by the main theorem. For the law above, the ratio of the variance contribution to the polynomial contribution in a Nagaev bound is
\[
 \frac{\nu^2(4\delta/T)^{1-2/p}\log(6/\delta)}
      {\nu^2T^{2/p-1}\delta^{-2/p}}
 =4^{1-2/p}\delta\log(6/\delta),
\]
which tends to zero as $\delta\downarrow0$. Consequently a logarithmic variance term alone cannot cover the polynomial confidence scale uniformly over this class. The construction changes its law with the horizon and confidence; it is not a fixed-law tail asymptotic. It also gives no lower bound forcing the distance-feedback stepsize in~\eqref{eq:main-condition}.

\section{Discussion and scope}\label{sec:discussion}
The results separate marginal stability, oracle complexity, and conditional confidence. Under marginal second-moment growth, the descent--displacement argument gives the finite-horizon BG-0 stochastic complexity and, through Doob's inequality, the expected squared maximum displacement $20\alpha T\Delta+10h_T\sigma^2$. The decreasing-stepsize extension controls the total weighted gradient and noise energies of a single run. It gives a near-$T^{-1/3}$ expectation rate and almost-sure convergence of the gradient norms to zero without bounded iterates.

Conditional higher moments have a different role. Centering the squared-noise energy removes a logarithm from its variance baseline. Stopping objective plus accumulated gradient energy also replaces a maximum-coefficient variance budget by a sum-of-squares budget. Together these give the refined variance contribution
$L\alpha\sigma^2+\sigma^2\ell/T$ while preserving the polynomial rare-shock term. The revised confidence ranges then attain $T^{-1/3}$ under distance growth and $T^{-1/2}$ under objective-gap growth at failure probability $T^{-r}$ for $r\le p/2-1$, without a logarithmic rate loss. Explicit baseline theorems remain valid for all horizons; the simplified refined statements additionally require $L\alpha T\ge1$.

The second-moment BG-0 oracle model remains distinct from common-sample two-point stochastic smoothness. The latter is a sufficient but stronger condition and permits faster variance-reduced methods. Exact Hilbert-space centering improves its second-moment verification to $K_2=L_{G,2}$, and spectral matrix moments sharpen affine-noise examples. Objective-gap growth is another distinct structure. Its set-valued verification permits nonunique unbounded minimizers as long as the noise is controlled along that set; flat objective directions with growing noise need not satisfy it.

\paragraph{Scope.}
Finite-time rates concern average squared gradients, maximal displacement, and the displayed randomized-output conversions. The single-run theorem adds qualitative almost-sure last-iterate stationarity, but not a finite-time last-iterate rate or convergence to a specified point. Shrinking-confidence statements apply to individual fixed-horizon runs; they are not simultaneous confidence events for every prefix of the decreasing-stepsize process. The marginal variance relaxation applies to expectation results, not to the conditional concentration arguments. Dimension-free constants allow dimension dependence through the moment parameters. The BG-0 optimality comparison concerns second-moment stochastic oracle complexity; the finite-$p$ confidence exponents, constants, and feedback caps are sufficient, not minimax confidence claims. The sparse-shock law can depend on horizon and confidence. This theoretical paper reports no empirical rate estimates.

\paragraph{Directions suggested by the analysis.}
An exactly log-free horizon-free stationarity rate, a direct plain-SGD analysis exploiting stochastic two-point regularity, and data-driven estimates of the state-dependent noise coefficients remain useful questions. Improving the conservative numerical constants is separate from the demonstrated removal of logarithmic rate losses. A confidence sequence valid for one decreasing-stepsize run would also require an additional uniform-in-time argument.

\subsection{Acknowledgment of supplied related work}\label{sec:provenance}
The deterministic-envelope argument in Section~\ref{sec:uniform} and the quadratic sparse-shock construction in Section~\ref{sec:lower} were motivated by the supplied manuscript \emph{Nagaev Bounds for Nonconvex Momentum: Finite Moments and Temporal Dependence}, specifically its Theorems~2.1--2.2 and Appendices~A--B. These ingredients are reproduced with full proofs and are not counted among the new contributions. The elementary $1<q\le2$ martingale moment estimate used in Lemma~\ref{lem:centered-energy} also appears in Appendix~C.1 of that supplied manuscript; its proof is included here. The centered-energy refinement, cumulative-gradient stopping, and decreasing-stepsize results are derived explicitly in the present revision. The supplied version is anonymized; no author identity or publication status is assigned here. The mathematical arguments above use public concentration results and are self-contained apart from the expressly cited probability and oracle lower-bound theorems.

\end{document}